\documentclass[10pt,twocolumn,letterpaper]{article}

\usepackage[pagenumbers]{cvpr} 

\usepackage{graphicx}
\graphicspath{{figure/}}
\usepackage{amsmath}
\usepackage{amssymb}
\usepackage{booktabs}

\usepackage{amsmath}
\usepackage{amsfonts,amssymb}
\usepackage{booktabs}

\usepackage{algorithm}  
\usepackage{algorithmic}  
\usepackage{multirow}
\usepackage[table]{xcolor}
\usepackage{microtype}      

\def\pt{\phantom{0}}
\definecolor{grassgreen}{rgb}{0.1,0.8,0.2}

\definecolor{lightgray}{rgb}{0.85, 0.85, 0.85}

\newcommand{\squishlist}{
 \begin{list}{$\bullet$}
  { \setlength{\itemsep}{0pt}
     \setlength{\parsep}{0.3pt}
     \setlength{\topsep}{0.3pt}
     \setlength{\partopsep}{0pt}
     \setlength{\leftmargin}{0.5em}
     \setlength{\labelwidth}{0.5em}
     \setlength{\labelsep}{0.3em} } }

\newcommand{\squishend}{
  \end{list}  }

\newtheorem{theorem}{Theorem}[section]

\newtheorem{lemma}[theorem]{Lemma}

\newenvironment{proof}{{\noindent\it Proof.}\quad}{\hfill $\square$\par}

\usepackage[pagebackref,breaklinks,colorlinks]{hyperref}
\usepackage[accsupp]{axessibility}

\usepackage[capitalize]{cleveref}
\crefname{section}{Sec.}{Secs.}
\Crefname{section}{Section}{Sections}
\Crefname{table}{Table}{Tables}
\crefname{table}{Tab.}{Tabs.}

\def\cvprPaperID{4802} 
\def\confName{CVPR}
\def\confYear{2023}

\begin{document}

\title{Three Guidelines You Should Know for Universally Slimmable Self-Supervised Learning}

\author{
 Yun-Hao Cao$^1$, \quad
 Peiqin Sun$^2$\thanks{Corresponding author.}, \quad
 Shuchang Zhou$^{2}$ \\\\
$^1$State Key Laboratory for Novel Software Technology, Nanjing University \\
$^2$MEGVII Technology \\
{\tt\small caoyh@lamda.nju.edu.cn}, {\tt\small \{sunpeiqin, zsc\}@megvii.com}
}

\maketitle


\begin{abstract}
We propose universally slimmable self-supervised learning (dubbed as US3L) to achieve better accuracy-efficiency trade-offs for deploying self-supervised models across different devices. We observe that direct adaptation of self-supervised learning (SSL) to universally slimmable networks misbehaves as the training process frequently collapses. We then discover that temporal consistent guidance is the key to the success of SSL for universally slimmable networks, and we propose three guidelines for the loss design to ensure this temporal consistency from a unified gradient perspective. Moreover, we propose dynamic sampling and group regularization strategies to simultaneously improve training efficiency and accuracy. Our US3L method has been empirically validated on both convolutional neural networks and vision transformers. With only once training and one copy of weights, our method outperforms various state-of-the-art methods (individually trained or not) on benchmarks including recognition, object detection and instance segmentation. Our code is available at \url{https://github.com/megvii-research/US3L-CVPR2023}.
\end{abstract}

\section{Introduction}
Deep supervised learning has achieved great success in the last decade, but the drawback is that it relies heavily on a large set of annotated training data. Self-supervised learning (SSL) has gained popularity because of its ability to avoid the cost of annotating large-scale datasets. Since the emergence of contrastive learning~\cite{simclr:hinton:ICML20}, SSL has clearly gained momentum and several recent works~\cite{mocov2:xinlei:arxiv2020,byol:grill:NIPS20} have achieved comparable or even better performance than the supervised pretraining when transferring to downstream tasks. However, it remains challenging to deploy trained models for edge computing purposes, due to the limited memory, computation and storage capabilities of such devices.

To facilitate deployment, several model compression techniques have been proposed, including lightweight architecture design~\cite{mobilenetv2:sabdker:CVPR18}, knowledge distillation~\cite{distillation:hinton:arxiv2015}, network pruning~\cite{deepcompression:han:ICLR16}, and quantization~\cite{dorefa:zhou:arxiv2016}. Among them, structured network pruning~\cite{thinet:luo:ICCV17} is directly supported and accelerated by most current hardware and therefore the most studied. However, most structured pruning methods require fine-tuning to obtain a sub-network with a specific sparsity, and a single trained model cannot achieve instant and adaptive accuracy-efficiency trade-offs across different devices. To address this problem in the context of supervised learning, the family of slimmable networks (S-Net) and universally slimmable networks (US-Net)~\cite{slimmable:yu:ICLR19, universal-slimmable:yu:ICCV19,once-for-all:hansong:ICLR20,dynamic-slimmable:Li:CVPR21} were proposed, which can switch freely among different widths by training only once. 

\begin{table}[t]
	\centering
	\caption{Comparisons between supervised classification and SimSiam under S-Net on CIFAR-100. The accuracy for SimSiam is under linear evaluation. `-' denotes the model collapses.}
	\label{tab:motivation}
	\setlength{\tabcolsep}{2.5pt}
	\renewcommand{\arraystretch}{0.75}
	\renewcommand{\multirowsetup}{\centering}
	\begin{tabular}{l|c|c|c|c|c}
	\hline
	\multirow{2}{*}{Type} & \multirow{2}{*}{Method} & \multicolumn{4}{c}{Accuracy (\%)} \\
	\cline{3-6}
	&&1.0x& 0.75x& 0.5x & 0.25x \\
	\hline
	\multirow{3}{*}{Supervised}& Individual & 73.8	& 72.8	&71.4&	67.3	 \\
	& S-Net~\cite{slimmable:yu:ICLR19} & 71.9&	71.7&	70.8	& 66.2	 \\
	& S-Net+Distill~\cite{universal-slimmable:yu:ICCV19} & 73.1&	71.9&	70.5&	67.2\\
	\hline
	\multirow{4}{*}{SimSiam~\cite{simsiam:kaiming:cvpr2021}}& Individual & 65.2	& 64.0&	60.6&	51.2	\\
	& S-Net~\cite{slimmable:yu:ICLR19} & - & - & - & - \\
	& S-Net+Distill~\cite{universal-slimmable:yu:ICCV19} & 46.9	& 46.9&	46.7&	45.3\\
	\cline{2-6}
	& Ours & \textbf{65.5} &\textbf{65.3} &\textbf{63.2} & \textbf{59.7}\\
	\hline
    \end{tabular}
\end{table}

Driven by the success of slimmable networks, a question arises: Can we train a \textit{self-supervised model} that can run at arbitrary width? A na\"ive solution is to replace the supervised loss with self-supervised loss based on the US-Net framework. However, we find that this solution doesn't work directly after empirical studies. Table~\ref{tab:motivation} shows that the phenomenon in self-supervised scenarios is very different. The model directly collapses after applying the popular SSL method SimSiam~\cite{simsiam:kaiming:cvpr2021} to slimmable networks~\cite{slimmable:yu:ICLR19}. Although using inplace distillation~\cite{universal-slimmable:yu:ICCV19} for sub-networks prevents the model from collapsing, there is still a big gap between the results of S-Net+Distill and training each model individually for SimSiam. So why is the situation so different in SSL and how to further improve the performance (i.e., close the gap)?

In this paper, we present a unified perspective to explain the differences and propose corresponding measures to bridge the gap. From a unified gradient perspective, we find that the key is that the guidance to sub-networks should be consistent between iterations, and we analyze which components of SSL incur the temporal inconsistency problem and why US-Net works in supervised learning. Based on these theoretical analyses, we propose three guidelines for the loss design of US-Net training to ensure temporal consistency. As long as one of them is satisfied, US-Net can work well, no matter in supervised or self-supervised scenarios. Moreover, considering the characteristics of SSL and the deficiencies of US-Net, we propose dynamic sampling and group regularization to reduce the training overhead while improving accuracy. Our main contributions are:
\squishlist
    \item We discover significant differences between supervised and self-supervised learning when training US-Net. Based on these observations, we analyze and summarize three guidelines for the loss design of US-Net to ensure temporal consistency from a unified gradient perspective. 
    \item We propose a dynamic sampling strategy to reduce the training cost without sacrificing accuracy, which eases coping with the large data volumes in SSL.
    \item We analyze how the training scheme of US-Net limits the model capacity and propose group regularization as a solution by giving different freedoms to different channels.
    \item We validate the effectiveness of our method on both CNNs and Vision Transformers (ViTs). Our method requires only once training and a single model, which can exceed the results of training each model individually, and is comparable to knowledge distillation from pretrained teachers.
\squishend

\section{Related Works}
\noindent\textbf{Self-supervised Learning.} To avoid time-consuming and expensive data annotations, many self-supervised methods were proposed to learn visual representations from large-scale unlabeled images or videos~\cite{jiasaw:mehdi:ECCV16,deepclustering:caron:ECCV18}. As the driving force of state-of-the-art SSL methods, contrastive learning methods greatly improve the performance of representation learning in recent years~\cite{InfoNCE:arxiv2018}. Contrastive learning is a discriminative approach that aims at pulling similar samples closer and pushing diverse samples far from each other. SimCLR~\cite{simclr:hinton:ICML20} and MoCo~\cite{moco:kaiming:CVPR20} both employ a contrastive loss function InfoNCE~\cite{InfoNCE:arxiv2018}, which requires negative samples. BYOL~\cite{byol:grill:NIPS20} and SimSiam~\cite{simsiam:kaiming:cvpr2021} discard negative sampling in contrastive learning by using an asymmetrical design. 

To improve the accuracy-efficiency trade-off for self-supervised models, many works have been proposed. Fang \textit{et al.}~\cite{seed:fang:ICLR21} proposed self-supervised knowledge distillation (SEED) for SSL with lightweight models. However, models at different widths (sparsities) must be trained individually, which incurs significant computational and storage overhead and is unsustainable for large data volumes. Moreover, it requires a pretrained teacher model while ours does not. Recently, SSQL~\cite{ssql:cao:ECCV2022} proposes to pretrain quantization-friendly self-supervised models to facilitate downstream deployment. Concurrent work DATA~\cite{DATA:chang:CVPR2022} proposes a neural architecture search (NAS) approach specialized for SSL. In contrast, we focus on structured pruning and we provide a unified theoretical explanation for the loss design of once-for-all training. 

\noindent\textbf{Slimmable Networks.} Slimmable networks~\cite{slimmable:yu:ICLR19} are widely studied because of their ability to execute at different widths, permitting instant and adaptive accuracy-efficiency trade-offs at runtime. Later, \cite{universal-slimmable:yu:ICCV19} proposes universally slimmable networks (US-Net), which extend slimmable networks to run at arbitrary width. Follow-up work OFA~\cite{once-for-all:hansong:ICLR20} extends the sampling space of sub-networks to depth and kernel size dimensions but it also inherits the loss design of US-Net by combining base loss and inplace distillation loss. \cite{ViTSlim:chavan:CVPR2022} explores finding an optimal sub-model from a vision transformer~\cite{vit:dosovitskiy:ICLR21}. They were all done, however, under the supervised learning paradigm, whereas our method is self-supervised.

\begin{figure*}[t]
	\centering
	\subfloat[General Framework]{
	\label{fig:network-a}
	\includegraphics[width=0.9\linewidth]{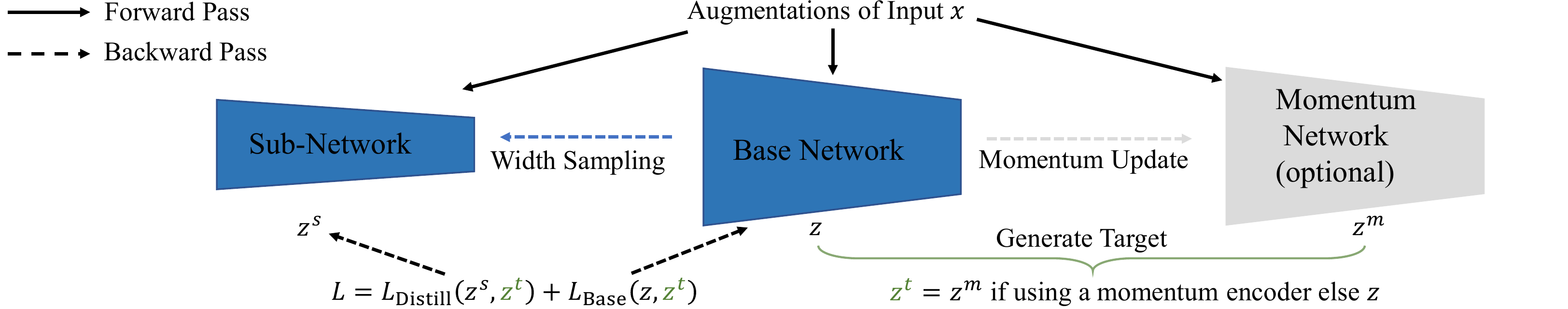} } \\ 
	\subfloat[Guidelines for Loss Design]{
	\label{fig:network-b}
	\includegraphics[width=0.3\linewidth]{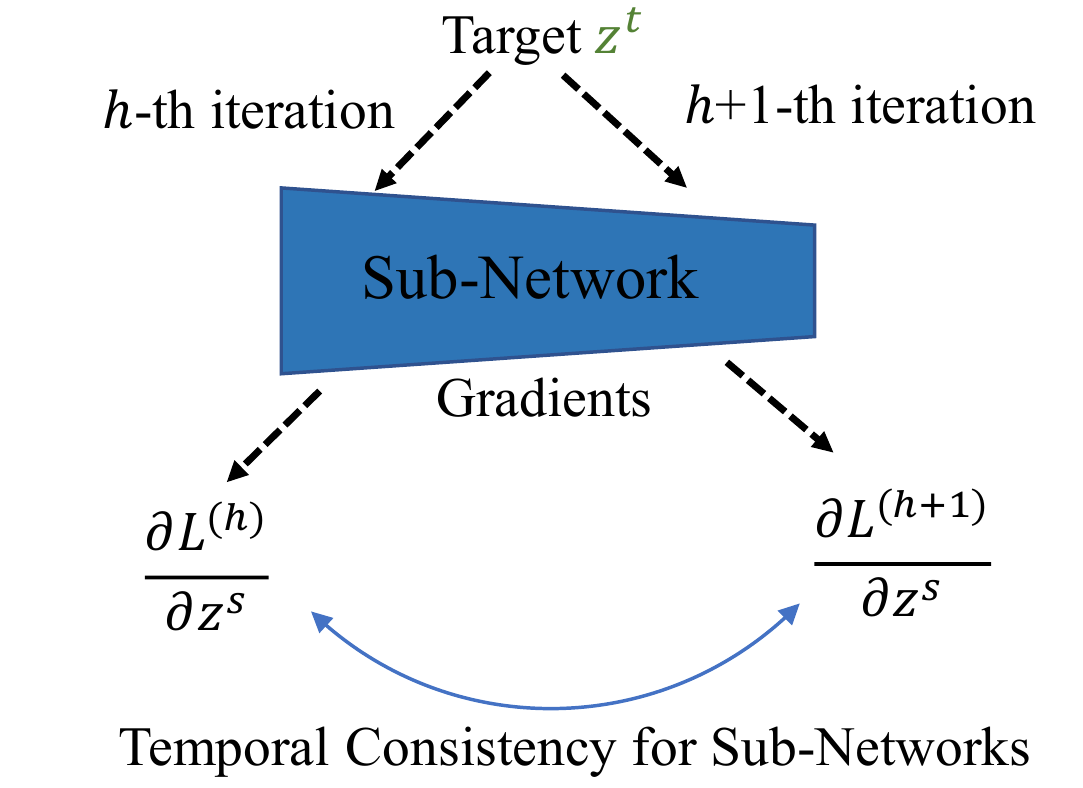}}
	\quad
	\subfloat[Dynamic Sampling]{
	\label{fig:network-c}
	\includegraphics[width=0.3\linewidth]{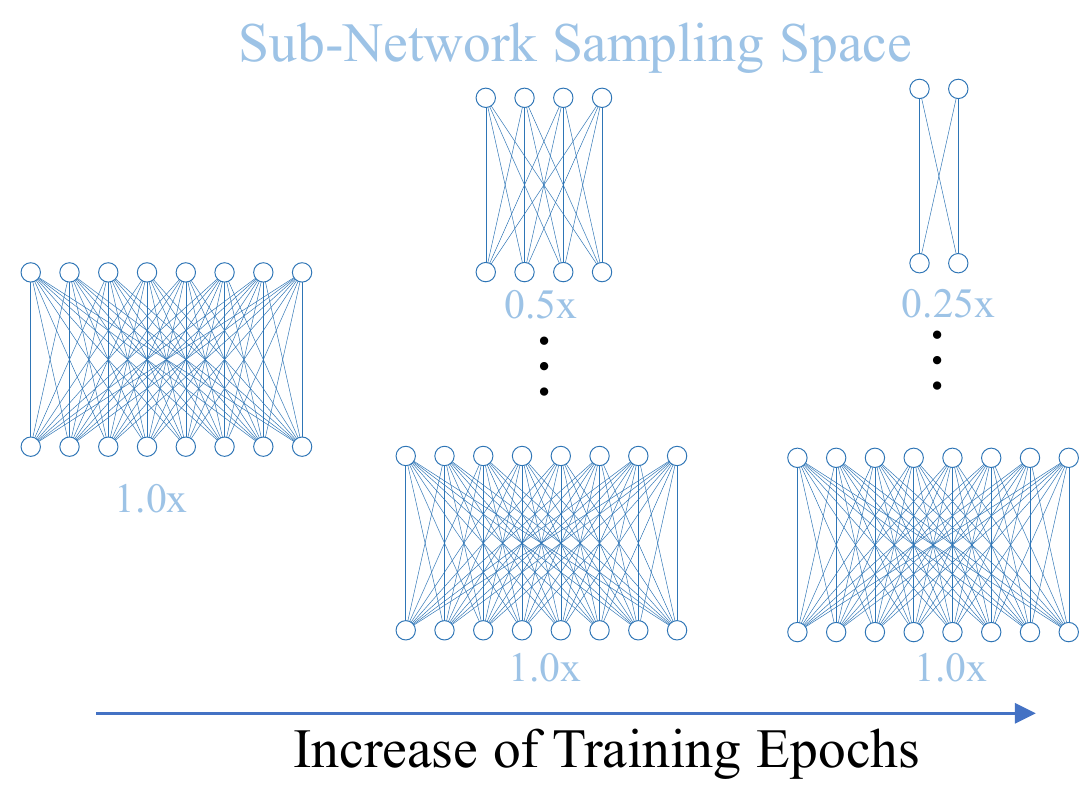}}
	\quad
	\subfloat[Group Regularization]{
	\label{fig:network-d}
	\includegraphics[width=0.3\linewidth]{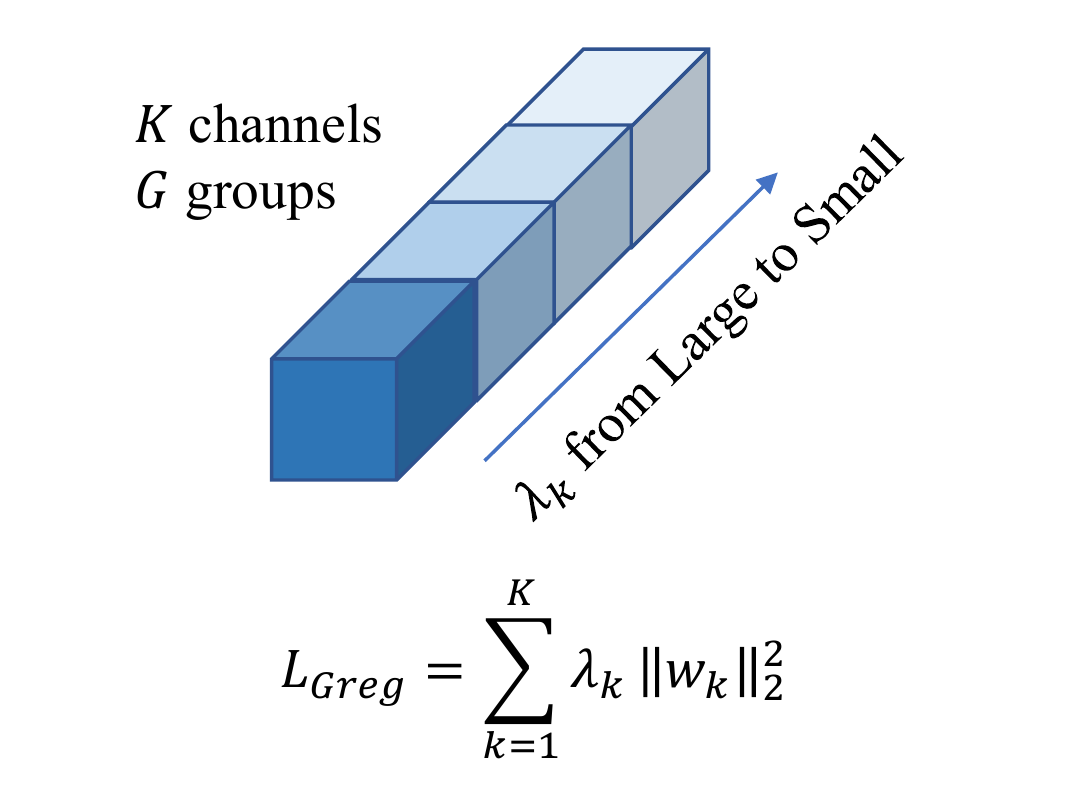}}
	\caption{The proposed framework and our method for universally slimmable self-supervised learning.}
	\label{fig:network}
\end{figure*}

\section{Method}
In this section, we begin with the notations and a brief review of previous works in Sec.~\ref{sec:preliminary}. Then, we introduce our method in Sec.~\ref{sec:method}, which we call Universally Slimmable Self-Supervised Learning (dubbed as US3L), as shown in Fig.~\ref{fig:network}. Finally, we show that temporal consistency of guidance is critical to the success of US-Net training by analyzing the stability of gradient updates of both self-supervised and supervised losses, and we propose three guidelines for the loss design to ensure this consistency in Sec.~\ref{sec:discussion}.

\subsection{Preliminary}\label{sec:preliminary}
In this subsection, we introduce two representative SSL methods SimSiam and SimCLR, as well as (universally) slimmable networks as preliminaries.

1) Self-supervised Losses. Let $\boldsymbol{x}_{i,1}$ and $\boldsymbol{x}_{i,2}$ denote two randomly augmented views from an input image $\boldsymbol{x}_i$. Let $f$ denote an encoder network consisting of a backbone (e.g., ResNet~\cite{resnet:he:CVPR16}) and a projection MLP head~\cite{simclr:hinton:ICML20}.

SimSiam~\cite{simsiam:kaiming:cvpr2021} maximizes the similarity between two augmentations of one image. A prediction MLP head~\cite{byol:grill:NIPS20}, denoted as $h$, transforms the output of one view and matches it to the other view. The output vectors for $\boldsymbol{x}_{i,1}$ are denoted as $\boldsymbol{z}_{i,1}\triangleq f(\boldsymbol{x}_{i,1})$ and $\boldsymbol{p}_{i,1}\triangleq h(f(\boldsymbol{x}_{i,1}))$, and $\boldsymbol{z}_{i,2}$ and $\boldsymbol{p}_{i,2}$ are defined similarly. The negative cosine similarity is defined as $D(\boldsymbol{p}, \boldsymbol{z})\triangleq -\frac{\boldsymbol{p}}{\Vert \boldsymbol{p} \Vert_2}\cdot \frac{\boldsymbol{z}}{\Vert \boldsymbol{z} \Vert_2}$ and we assume both $\boldsymbol{z}$ and $\boldsymbol{p}$ have been $L_2$-normalized for simplicity in 
subsequent discussions. Let $SG(\cdot)$ denote the stop-gradient operation. Then, the loss function in SimSiam is:
\begin{equation}
    \label{eq:simsiam}
    L_{\text{MSE}} = \sum_i D(\boldsymbol{p}_{i,1}, SG(\boldsymbol{z}_{i,2})) + D(\boldsymbol{p}_{i,2}, SG(\boldsymbol{z}_{i,1}))\,.
\end{equation}

SimCLR~\cite{simclr:hinton:ICML20} and MoCo~\cite{moco:kaiming:CVPR20} contrast with negative samples using InfoNCE~\cite{InfoNCE:arxiv2018} loss:
\begin{equation}
    \label{eq:infonce}
    L_{\text{NCE}} = - \sum_i \log \frac{e^{\boldsymbol{z}_{i,1} \cdot \boldsymbol{z}_{i,2}}}{e^{\boldsymbol{z}_{i,1} \cdot \boldsymbol{z}_{i,2}}+\sum_{j\neq i, v\in\{1,2\}} e^{\boldsymbol{z}_{i,1} \cdot \boldsymbol{z}_{j,v}}} \,,
\end{equation}
where we omit the temperature parameter $\tau$ for simplicity.

2) Slimmable Networks. Slimmable networks~\cite{slimmable:yu:ICLR19} are a class of networks that can be executable at different scales. During training, only the smallest, the largest and a few randomly sampled networks are used to calculate the loss in each iteration, which is known as the \textit{sandwich rule}. Further, inplace distillation~\cite{universal-slimmable:yu:ICCV19} is introduced to improve performance, where the knowledge inside the largest network is transferred to sub-networks by using distillation loss.

\subsection{The Proposed Method}\label{sec:method}
The following three subsections describe the components of our US3L method (Algorithm~\ref{algorithm}).

\subsubsection{Loss Design}
The general framework of our method is depicted in Fig.~\ref{fig:network-a}, in which the loss function is composed of base loss (for the base/largest network) and distillation loss (for sub-networks). By default, we use a momentum encoder to generate targets, InfoNCE as the base loss, and MSE as the distillation loss. Also, we show that we should use an auxiliary distillation head to mitigate the impact of the capacity difference between teacher and student. The overall loss function is 
\begin{equation}
\label{eq:full-loss}
    L = \underbrace{L_{\rm{NCE}}}_{L_{\rm{Base}}} \underbrace{-\sum_{i}\sum_{s} g(z_i^s) \cdot z_i^m}_{L_{\rm{Distill}}} \,,
\end{equation}
where $g(\cdot)$ is an auxiliary distillation MLP head, $z^s$ and $z^m$ are the output of the sub-network and momentum encoder, respectively. Notice that Eq.~\eqref{eq:full-loss} is not the only option for the loss design and it can work well as long as it satisfies our guidelines (Fig.~\ref{fig:network-b}), which will be discussed in Sec.~\ref{sec:discussion}.

\subsubsection{Dynamic Sampling}
It is worth noting that Yu \textit{et al.}~\cite{slimmable:yu:ICLR19} sampled four switches in each iteration throughout training, which is very time-consuming for SSL training. Therefore, we design a dynamic sampling strategy to reduce the training overhead while improving performance (Fig.~\ref{fig:network-c}). First, we argue that it is unnecessary to introduce the training of sub-networks at the beginning. We believe that a good and consistent teacher is essential for the learning of sub-networks~\cite{kd-patient:Beyer:CVPR22}, so we only need to train the base network at the start. Second, the training of sub-networks should be gradual. Specifically, the width of the smallest sub-network should be gradually reduced. By combining the two sampling strategies described above, we successfully reduce the sampling number $s$ from 4 to 3 (theoretical minimum sampling number) without performance drop (see appendix for detailed analysis and results). 

In our implementation, the training process is divided into two stages. In the first stage, only the largest network is trained (i.e., $s=1$). In the second stage, we sample the largest, the smallest plus a random width ($s=3$), and the width of the smallest model is gradually reduced. For example, the sampling width range in the second stage begins with [0.75,1.0], then [0.5,1.0], and finally [0.25,1.0].

\subsubsection{Group Regularization} \label{sec:groupreg}
Given two channels $k_1$ and $k_2$ ($k_1<k_2$), if $k_2$ is used in US-Net, then $k_1$ must also be used. In other words, the earlier channels are used more frequently than the later ones. Therefore, in the training of US-Net, the majority of the weights will be concentrated on the earlier channels to ensure the performance of sub-networks. However, such a weight distribution will limit the base model's capacity and thus affect its performance. To address this problem, we propose group regularization by giving more degrees of freedom (i.e., smaller regularization coefficients) to the later channels (Fig.~\ref{fig:network-d}), so that their weights are more fully utilized. We divide the total $K$ channels into $G$ groups in order, with each group containing $K_{G}=\lfloor K/G \rfloor$ channels. Then we define: 

\begin{equation}
    L_{\text{GReg}} = \sum_{k=1}^{K} \lambda_k \lVert \mathbf{w}_k \rVert _2^2 \,,
\end{equation}

\begin{equation}
\label{eq:greg_lambda}
    \lambda_k = \lambda (1- \lfloor k/K_G \rfloor \alpha) \,,
\end{equation}
where $\mathbf{w}$ denotes the weight matrix and we set $G=8$ and $\alpha=0.05$ throughout this paper. Notice that when $\alpha=0$, group regularization degenerates into the standard $L_2$ regularization. We also empirically demonstrate that group regularization is tailored for US-Net in the appendix.

\begin{algorithm}[t]
\caption{The proposed US3L method}
\label{algorithm}
\begin{algorithmic}[1]
\REQUIRE Define width range $R=[R_{\text{min}}, R_{\text{max}}]$x, for example, $R_{\text{min}}=0.25$, $R_{\text{max}}=1.0$. 
\FOR{$h=1,\dots, T_{iters}$}
    \STATE Define period length $T_p=\lfloor T_{iters}/4 \rfloor$.
    \STATE Clear gradients, \textit{optimizer.zero\_grad()}.
    \STATE Run base network $z_1=M(x_1), z_2=M(x_2)$.
    \STATE Compute base loss, $loss = L_{\rm{Base}}(z_1,z_2)+L_{\rm{GReg}}$.
    \STATE Detach label, $z^t_1=z_1.detach()$, $z^t_2=z_2.detach()$.
    \IF{$t\leq T_p$}
    \STATE Continue
    \ENDIF 
    \STATE Dynamic adjust range, $R_{\text{min}}=1-0.25\lfloor t/T_p \rfloor$.
    \STATE Randomly sample a width from $R$ as \textit{width samples}.
    \STATE Add the smallest width $R_{\text{min}}$ to \textit{width samples}.
    \FOR{\textit{width} in \textit{width samples}}
        \STATE Execute sub-network $M'$ at width, and distillation head $g$, $z^s_1=g(M'(x_1))$, $z^s_2=g(M'(x_2))$.
        \STATE $loss$ $\mathrel{+}=$ $L_{\rm{Distill}}(z^s_1, z^t_2)+L_{\rm{Distill}}(z^s_2, z^t_1)$.
    \ENDFOR
    \STATE Accumulate gradients, \textit{loss.backward().}
    \STATE Update weights, \textit{optimizer.step().}
\ENDFOR
\end{algorithmic}
\end{algorithm}

\subsection{Three Guidelines for Loss Design} \label{sec:discussion}
The special feature of US-Net training is the introduction of sub-network training (i.e., $L_{\text{Distill}}$), so the training stability of the sub-networks is very important. In this paper, we find that the key is to ensure the temporal consistency of guidance for sub-networks. One image has different views in two adjacent iterations, which will produce different outputs because the model has not converged and is unstable. We hope that the gradients generated by different views of \textit{the same image} will also be close between iterations (i.e., robust to changes and provide consistent guidance to sub-networks).

In the context of SSL, \eqref{eq:simsiam} and \eqref{eq:infonce} can also be adapted for distillation and we use $z^s$ and $z^t$ to denote the output of the sub and base network, respectively. The losses are:

\begin{equation}
    \label{eq:mse_distill}
    L_{\text{MSE-distill}} = - z_i^s \cdot z_i^t \,,
\end{equation}

\begin{equation}
    \label{eq:infonce_distill}
        L_{\text{NCE-distill}} = -z_i^s \cdot z_i^t + \log \sum_k e^{z_i^s \cdot z^t_{k}} \,.
\end{equation}

\begin{lemma}
    \label{lemma1}
    MSE is not robust to changes in the output, whereas NCE is stabilized by distances from other samples.
\end{lemma}

\begin{proof}
    For \eqref{eq:mse_distill}, the derivative can be derived as follows:
    \begin{equation}
        \frac{\partial L}{\partial z_i^s} = -z_i^t \,.
    \end{equation}

    For \eqref{eq:infonce_distill}, the derivative can be derived as follows: 
    \begin{equation}
        \frac{\partial L}{\partial z^s_i} = -z_i^t+\sum_j \frac{e^{z^s_i \cdot z^t_j}}{\sum_k e^{z^s_i \cdot z^t_k}} z^t_j
        \triangleq -z_i^t + \sum_{j} P_j z^t_j \,.
    \end{equation}
    Hence, we see that for MSE, the gradient only depends on the output $z^t$. This output will be very unstable in different iterations due to factors such as rapid model updates and image augmentations, resulting in temporal gradient instability. In contrast, NCE loss is also stabilized by the distance from other samples (corresponding to the extra $\sum_j P_j z^t_j$ term).
\end{proof}

To further illustrate Lemma~\ref{lemma1}, consider the following example (Fig.~\ref{fig:three-rules}b,c). Assume that due to image augmentations, all outputs are transformed by the same rotation matrix $w^{\theta}$ from the $h$-th to the $h$+1-th iteration (Fig.~\ref{fig:three-rules}a). The gradient difference between iterations for MSE is: 
\begin{equation}
    \label{eq:mse-diff}
   \frac{\partial L^{(h+1)}}{\partial z^s_i} - \frac{\partial L^{(h)}}{\partial z^s_i} = (I- w^{\theta})z^t_i \,,
\end{equation}
where $I$ denotes the identity matrix. For InfoNCE, we have: 
\begin{equation}
    \label{eq:nce-diff}
    \frac{\partial L^{(h+1)}}{\partial z^s_i} - \frac{\partial L^{(h)}}{\partial z^s_i} =  (I- w^{\theta})(z^t_i-\sum_j P_j z^t_j) \,.
\end{equation}

If the output of the student is already aligned with the teacher, then we have $P_i \approx 1$ and thus it can be verified that $
\lVert \frac{\partial L_{\text{NCE}}^{(h+1)}}{\partial z^s_i} - \frac{\partial L_{\text{NCE}}^{(h)}}{\partial z^s_i} \rVert_2 < \lVert \frac{\partial L_{\text{MSE}}^{(h+1)}}{\partial z^s_i} - \frac{\partial L_{\text{MSE}}^{(h)}}{\partial z^s_i} \rVert_2$. That is, NCE loss achieves better temporal stability than MSE for the learning of sub-networks.

\begin{figure}[t]
	\centering
	\includegraphics[width=\linewidth]{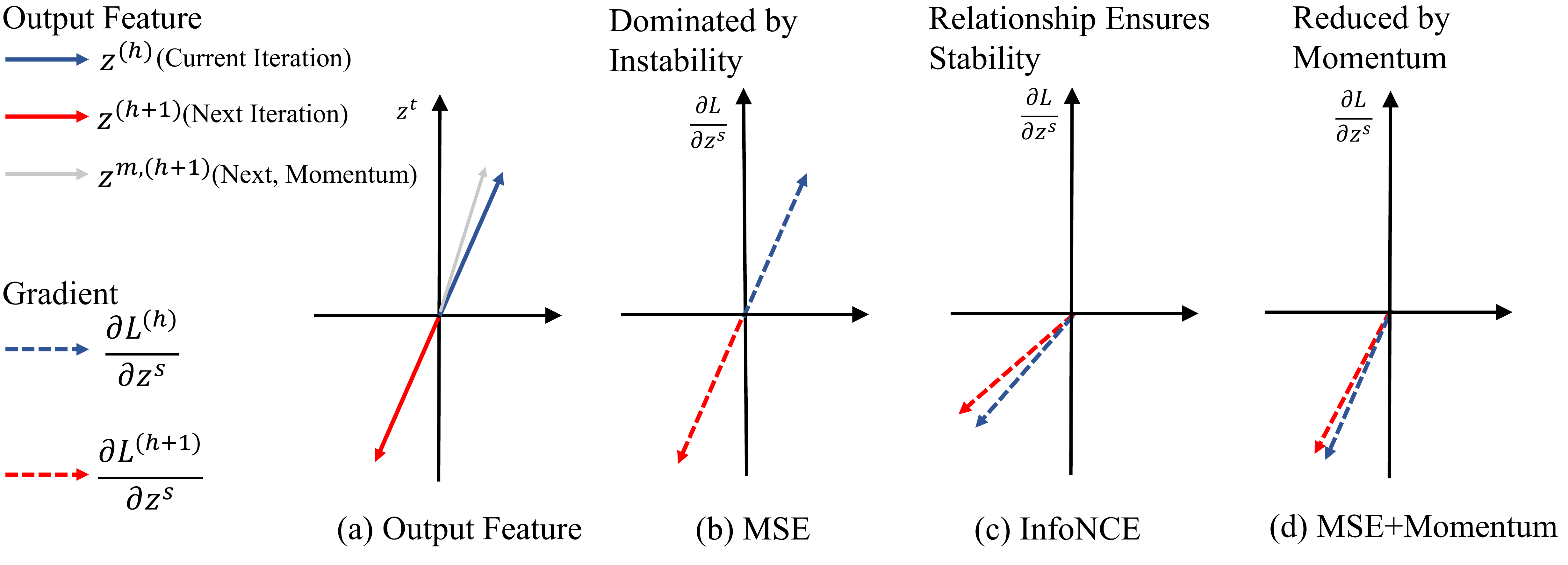}
	\caption{Illustration of feature changes and corresponding gradient changes under various settings. Best viewed in color.}
	\label{fig:three-rules}
\end{figure}

\begin{lemma}\label{lemma2}
    Supervised cross entropy (CE) is also stabilized by relative distance and temporal consistency is preserved.
\end{lemma}
\begin{proof}
    Let $y_i$ denote the label for $x_i$, $C$ denote the number of classes, $\mathbf{w}\in{\mathbb{R}^{d\times{C}}}$ denote the weight matrix of the classification head ($d$ is the feature dimension). Then:
    \begin{equation*}
    L_{\text{CE}}  = - \log\frac{e^{\mathbf{w}_{y_i}^T\boldsymbol{z}_i}}{\sum_{j=1}^Ce^{\mathbf{w}_j^T\boldsymbol{z}_i}} = - \mathbf{w}_{y_i}^T\boldsymbol{z}_i+\log\sum_{j=1}^C e^{\mathbf{w}_j^T\boldsymbol{z}_i} \,.
    \end{equation*}
     The derivative is as follows:
    \begin{equation*}
    \frac{\partial{L_{\text{CE}}}}{\partial{\boldsymbol{z}_i}}=-\mathbf{w}_{y_i}+\sum_{j=1}^C\frac{e^{\mathbf{w}_j^T\boldsymbol{z}_i}}{\sum_{k=1}^C e^{\mathbf{w}_k^T{\boldsymbol{z}_i}}}\mathbf{w}_j \triangleq -\mathbf{w}_{y_i}+\sum\nolimits_{j} P_j\mathbf{w}_j\,.
    \end{equation*}
    The gradient is not only related to the target class weight $\mathbf{w}_{y_i}$ and the analysis is then similar to the NCE above.
\end{proof}

From Lemma~\ref{lemma2} we can understand the huge difference between MSE-based SimSiam and CE-based supervised classification in Table~\ref{tab:motivation}. The key is that inconsistent outputs can make the temporal gradient updates for MSE very unstable.

\begin{lemma} \label{lemma3}
    A momentum teacher will better preserve temporal consistency by producing slowly updating outputs. \footnote{It is a known fact that a momentum teacher reduces the degree of output change and helps the model with more stable training~\cite{byol:grill:NIPS20}.}
\end{lemma}

From the above analyses, we can summarize three guidelines below to ensure the temporal consistency of guidance: 

1. The base loss is based on the relative distance to produce temporal consistent outputs of the base network.

2. The distillation loss is based on the relative distance to produce temporal consistent guidance for sub-networks.

3. A momentum teacher is used to produce stable guidance for sub-networks.

Experimental results in Sec.~\ref{sec:ablation} further verify the effectiveness of our proposed three guidelines, and we will empirically find that \textbf{at least one of the three guidelines needs to be satisfied} to make it work for the US-Net framework.

\section{Experimental Results}
We introduce the implementation details in Sec.~\ref{sec:details}. We experiment with CNNs in Sec.~\ref{sec:exp-cifar} and ViTs in Sec.~\ref{sec:exp-vit} on CIFAR-100~\cite{cifar} and CIFAR-10~\cite{cifar}, respectively. Then, we experiment on ImageNet~\cite{ILSVRC2012:russakovsky:IJCV15} (IN) in Sec.~\ref{sec:exp-imagenet} and we evaluate the transfer performance of ImageNet pretrained models on downstream recognition, object detection, and instance segmentation benchmarks. Finally, we investigate the effects of different components in our method in Sec.~\ref{sec:ablation}.

\begin{table*}[!ht]
	\centering
	\small
	\caption{Main results on CIFAR-100. `-' denotes the model collapses. $n$ denotes the number of sub-models and $n=9$ in this table. $T$ denotes the cost of training one model on CIFAR-100 for 400 epochs and we do not consider the effect of model size here, because models of different sizes are encountered in each method. The best two results are \textbf{bolded} and \underline{underlined}, respectively.}
	\label{tab:main-cifar100}
	\setlength{\tabcolsep}{2.5pt}
	\renewcommand{\arraystretch}{0.75}
	\renewcommand{\multirowsetup}{\centering}
	\begin{tabular}{l|c|c|c|c|c|c|c|c|c|c|c|c|c}
		\hline
		\multirow{2}{*}{Backbone}& \multirow{2}{*}{Method}& Once& Pretrained & Training& \multicolumn{9}{c}{Linear Accuracy (\%)} \\
		\cline{6-14}
		& &  Training&Teacher&Cost&1.0x&0.9x&0.8x&0.7x&0.6x&0.5x&0.4x&0.3x&0.25x\\
		\hline
        \multirow{10}{*}{ResNet-18}
		&SimCLR~\cite{simclr:hinton:ICML20} & $\times$ &  $\times$ & $nT$ & 66.5	& 65.4	& 64.7&	63.7&	62.6&	61.0&	59.0&	56.1	&53.6 \\
		&SimSiam~\cite{simsiam:kaiming:cvpr2021} & $\times$  & $\times$& $ nT $ & 66.5	& 65.4	& 64.6&	63.5&	62.6&	60.0&	58.3&	54.9	&52.4\\
		&BYOL~\cite{byol:grill:NIPS20} & $\times$  & $\times$& $nT$ & 66.8	& 66.0	& 65.6	& 65.3&	63.0&	62.1&	59.5	&56.0	& 54.3\\
		&SEED~\cite{seed:fang:ICLR21} & $\times$  & BYOL R-50 &$nT$ & 67.3 & 66.6 &65.8&65.2&64.8&63.5&62.2&60.1&58.5 \\
		& SEED-MSE & $\times$  & BYOL R-18& $nT$& 67.5	& 67.2	& 66.5	& 66.0	& 65.9	& 64.8	& \underline{64.0}	& \underline{62.4} & 60.1\\
		& SEED-MSE & $\times$  & BYOL R-50&$nT$ & 67.5	&66.8&	66.7&	66.0&	65.4&	\underline{64.9}	&63.6&	61.3&	60.1\\
		& US~\cite{universal-slimmable:yu:ICCV19}+SimCLR & \checkmark &$\times$ & $4T$ &65.5&	64.9&	63.8&	63.6&	62.7	&61.8&	60.2&	58.2&	57.4 \\
		& US~\cite{universal-slimmable:yu:ICCV19}+SimSiam  & \checkmark & $\times$ & $4T$ & 57.5&57.4&57.3&57.0&56.3&55.4&54.5&53.1&52.4\\
		&Ours & \checkmark &$\times$ & $\mathbf{2.5}\boldsymbol{T}$ & \underline{69.0}&\underline{68.2}&\underline{68.0}&\underline{66.9}&\underline{66.1}& 64.7 & 62.6 & 60.9  & \underline{60.4} \\
		& Ours (800ep) &  \checkmark &$\times$ & $5T$& \textbf{70.1}&\textbf{69.3}&\textbf{69.0}&\textbf{68.7}&\textbf{67.3}&\textbf{66.4}&\textbf{64.2}&\textbf{63.1}&\textbf{62.3}\\
        \hline
        \multirow{7}{*}{ResNet-50}
		&BYOL~\cite{byol:grill:NIPS20} & $\times$  & $\times$ &$nT$ & 67.0	& 66.7	& 66.5	& 66.3	& 66.0	& 64.9	& 63.8	& 62.1	& 61.2 \\
		&SEED~\cite{seed:fang:ICLR21} & $\times$  & BYOL R-50 &$nT$& 70.3 & 69.8 &69.6 & 69.4 & 69.0 & 68.2 & 67.2 & 65.6 & 65.1\\
		& SEED-MSE & $\times$  & BYOL R-50 &$nT$& 69.4	& 69.0	& 68.5&	69.1&	68.4&	68.1	&67.3&	66.9&	66.4 \\
		& US~\cite{universal-slimmable:yu:ICCV19}+SimCLR  & \checkmark &$\times$ &$4T$& 70.1 & 69.9 & 69.7 & 69.3 & 68.7 & 68.2 & 67.5 & 66.0 &  65.5 \\
		& US~\cite{universal-slimmable:yu:ICCV19}+SimSiam  & \checkmark & $\times$ &$4T$ & 54.7 & 54.6 & 54.7 & 54.7 & 54.7 & 54.8 & 54.6 & 54.3 & 54.0 \\
		&Ours & \checkmark &$\times$ & $\mathbf{2.5}\boldsymbol{T}$ & \underline{72.6} & \underline{72.0} &  \underline{71.5} & \underline{71.2} & \underline{70.6} & \underline{70.2} &  \underline{68.6} & \underline{67.7} & \underline{67.4} \\
		& Ours (800ep) &  \checkmark &$\times$ &$5T$ & \textbf{73.0}&\textbf{72.5}&\textbf{71.9}&\textbf{71.6}&\textbf{71.1}&\textbf{70.8}&\textbf{69.1}&\textbf{68.0}&\textbf{67.6} \\
        \hline
        \multirow{6}{*}{MobileNetv2}
		&BYOL~\cite{byol:grill:NIPS20} & $\times$  & $\times$ &$nT$ & 61.2	& 60.7&	60.5&	60.2&	59.9&	58.7&	57.3	&54.6&	51.9 \\
		& SEED-MSE & $\times$  & BYOL R-50 &$nT$ & \textbf{68.6}	& \textbf{68.9} &	\textbf{67.6} &	\textbf{67.3} &	\textbf{67.4}	& \textbf{66.3}	& \textbf{65.5} &	\textbf{64.0} &	\textbf{62.6}  \\
		& SEED-MSE & $\times$  & BYOL MBv2 & $nT$ & 63.8	&63.5&	63.8&	\underline{63.6} &	\underline{63.6} &	\underline{63.3}	&\underline{62.7}&	\underline{62.1}&	\underline{59.8}  \\
		& US~\cite{universal-slimmable:yu:ICCV19}+SimCLR  & \checkmark &$\times$ & $4T$ &  56.2 & 56.0 & 55.3 & 55.0 & 54.8 & 54.3 & 54.0 & 53.2 &  52.2 \\
		& US~\cite{universal-slimmable:yu:ICCV19}+SimSiam  & \checkmark & $\times$ &$4T$ & - & - & -& - & - & - & - & - &-  \\
		&Ours & \checkmark &$\times$ &$\mathbf{2.5}\boldsymbol{T}$ &  \underline{65.7} & \underline{65.1} & \underline{64.2} & \underline{63.6} &  63.4 & 62.2 & 61.5 & 60.7 & 59.3  \\
        \hline
	\end{tabular}
\end{table*}

\subsection{Implementation Details}\label{sec:details}
\noindent\textbf{Datasets.} The main experiments are conducted on three benchmark datasets, i.e., CIFAR-10, CIFAR-100 and ImageNet. We also conduct transfer experiments on 5 recognition benchmarks as well as 2 detection benchmarks Pascal VOC 07\&12~\cite{VOC:mark:IJCV10} and COCO2017~\cite{coco:LinTY:ECCV14}.

\noindent\textbf{Backbones.} In addition to the commonly used ResNet-50~\cite{resnet:he:CVPR16}, we also adopt 2 smaller networks, i.e., ResNet-18 and MobileNetv2~\cite{mobilenetv2:sabdker:CVPR18}. Moreover, we evaluate our method on vision transformers~\cite{vit:dosovitskiy:ICLR21}. Sometimes we abbreviate ResNet-18/50 to R-18/50 and MobileNetv2 to MBv2.

\noindent\textbf{Training details.} We use SGD for pretraining, with a batch size of 512 and a base lr=0.5. The learning rate has a cosine decay schedule. The weight decay is 0.0001 and the SGD momentum is 0.9. We pretrain for 400 epochs on CIFAR-100 and 100 epochs on ImageNet unless otherwise specified.

\subsection{Experiments on CIFAR}\label{sec:exp-cifar}
We compare our method with state-of-the-art SSL methods BYOL~\cite{byol:grill:NIPS20}, SimSiam~\cite{simsiam:kaiming:cvpr2021} and SimCLR~\cite{simclr:hinton:ICML20}, and \textit{individually} pretrain for them to obtain models at different channel widths. Notice that it is not a fair comparison with our method because they must pretrain different models for different widths (i.e., need 9 pretrained models for 9 widths), whereas ours is \textit{trained only once with one copy of weights}. To better illustrate the effectiveness of our method, we also compare it with one strong baseline SEED~\cite{seed:fang:ICLR21}. In addition to training each model separately, SEED also requires an additional pretrained teacher model. The original implementation of SEED uses InfoNCE-based distillation loss and we add one more variant SEED-MSE (use MSE distillation loss). We also directly adapt different SSL methods to US-Net~\cite{universal-slimmable:yu:ICCV19} for comparison. We report the linear evaluation accuracy for pretrained models of all methods under the same setting.

As shown in Table~\ref{tab:main-cifar100}, our method achieves higher accuracy consistently than baseline methods. Take R-18 as an example, our method achieves \textbf{2.2\%} and \textbf{6.1\%} higher accuracy than BYOL at widths of 1x and 0.25x, respectively, with less training cost. When compared with the strong baseline SEED, our method even performs better in most cases, with much less training cost and additional dependencies. Also notice that the SEED-MSE variant performs even better than the original SEED, especially for small subnets, which is consistent with our analysis in Sec.~\ref{sec:discussion}. That is, when we have a pretrained teacher which can already generate consistent targets, it is sufficient to use MSE for distillation. We can observe similar trends and improvements for R-50. For MobileNetv2, our method outperforms individually trained BYOL and SEED-MSE (when using pretrained MBv2 as the teacher). However, when we use a better teacher (e.g., R-50), our method is inferior to SEED (the gap is within 3 points). Also, our method outperforms the US-Net baseline at all widths for all three backbones (the role of each component will be discussed in Sec.~\ref{sec:ablation}). Moreover, we can see that our method greatly reduces the training cost. The training time for individually-trained methods is proportional to the number of sub-networks $n$ that need to be used while ours is not affected by $n$. When compared with the original US-Net, we reduced the expected number of sampling in each iteration from 4 to 2.5 (see appendix for more analyses).

In conclusion, our method outperforms various individually trained SSL algorithms and even achieves comparable accuracy with knowledge distillation, by only training once.

\begin{table}[t]
	\centering
	\caption{Linear evaluation results for ViT on CIFAR-10.}
	\label{tab:main-vit-cifar10}
	\small
	\renewcommand{\arraystretch}{0.75}
	\setlength{\tabcolsep}{2.5pt}
	\renewcommand{\multirowsetup}{\centering}
	\begin{tabular}{l|c|c|c|c|c|c}
		\hline
		\multirow{2}{*}{Backbone}& \multirow{2}{*}{Method}& Once  & \multicolumn{4}{c}{Linear Accuracy (\%)} \\
		\cline{4-7}
		& &  Training &1.0x &0.75x & 0.5x & 0.25x\\
		\hline
		\multirow{3}{*}{ViT-Tiny}
		&MoCov3~\cite{mocov3:chen:ICCV21} & $\times$  & 82.6 & 79.5 & 75.8 & 68.0  \\
		&US+MoCov3   & \checkmark  & 79.8 &  79.4 & 77.6 & 76.4 \\
		&Ours   & \checkmark  & \textbf{86.0} &\textbf{84.7} & \textbf{83.3} &\textbf{80.2}  \\
         \hline
        \multirow{3}{*}{ViT-Small}
		&MoCov3~\cite{mocov3:chen:ICCV21} & $\times$  &  88.0 &  86.8 & 83.0 &  75.5  \\
		&US+MoCov3  & \checkmark  &  88.2 & 87.5 & 86.3 & 84.9 \\
		&Ours   & \checkmark  & \textbf{90.3} & \textbf{89.7} & \textbf{88.7} & \textbf{85.5}  \\
		\hline
	\end{tabular}
\end{table}

\subsection{Experiments with Vision Transformer}\label{sec:exp-vit}
To further demonstrate the efficacy of our method, we evaluate our method on vision transformers~\cite{vit:dosovitskiy:ICLR21} (ViTs) and adopt the popular MoCov3~\cite{mocov3:chen:ICCV21} as our baseline method. We employ the official code and experiment on CIFAR-10. To the best of our knowledge, we are the first to study slimmable self-supervised ViTs and we directly reduce the embedding dimension in all layers. As shown in Table~\ref{tab:main-vit-cifar10}, our method significantly outperforms individually trained MoCov3. Also, our method surpasses US+MoCov3 (adapt MoCov3 to US-Net) at all widths despite using less training cost. In conclusion, the results show that our method still works even for complex architectures such as vision transformers.

\subsection{ImageNet and Transferring Experiments}\label{sec:exp-imagenet}

In this subsection, we do unsupervised pretraining on the large-scale ImageNet training set without using labels. The linear evaluation results on ImageNet are shown in Table~\ref{tab:main-imagenet}. Also, we evaluate the transfer ability of the learned representations on ImageNet later. We train one model for our method and 4 separate models for BYOL, each of which is trained for 100 epochs. As shown in Table~\ref{tab:main-imagenet}, our US3L achieves higher accuracy than BYOL at all widths and our advantages become greater as the width shrinks.

\begin{table}[t]
	\centering
	\caption{Linear evaluation results on ImageNet.}
	\label{tab:main-imagenet}
	\small
	\renewcommand{\arraystretch}{0.75}
	\setlength{\tabcolsep}{2pt}
	\renewcommand{\multirowsetup}{\centering}
	\begin{tabular}{l|c|c|c|c|c|c}
		\hline
		\multirow{3}{*}{Backbone}& \multirow{3}{*}{Method}& Once  & \multicolumn{4}{c}{Linear Accuracy (\%)} \\
		\cline{4-7}
		& &  Training &1.0x &0.75x & 0.5x & 0.25x\\
		\hline
        \multirow{3}{*}{ResNet-18}
		&BYOL & $\times$  & 54.0 & 53.7 & 47.4 & 34.9 \\
		&US+BYOL & \checkmark & 55.9 & 53.1 & 48.0 & 40.6  \\
		&Ours   & \checkmark  & \textbf{56.9} &  \textbf{54.5} & \textbf{48.7} & \textbf{40.7} \\
        \hline
        \multirow{3}{*}{ResNet-50}
		&BYOL & $\times$   & 68.1 &66.3 &61.2 &50.9 \\
		&US+BYOL & \checkmark & 64.7 & 64.3 & 62.6 & 57.1  \\
		&Ours& \checkmark & \textbf{68.4} & \textbf{66.7} &\textbf{63.4} &\textbf{57.7}\\
        \hline
	\end{tabular}
\end{table}

We investigate the downstream object detection performance on Pascal VOC07\&12 in Fig.~\ref{fig:voc} and COCO2017 in Table~\ref{tab:coco}. The detector is Faster R-CNN~\cite{faster-rcnn:ren:NIPS15} for VOC and Mask R-CNN~\cite{mask-rcnn:he:ICCV17} for COCO (both with FPN~\cite{FPN:kaiming:CVPR17} backbone), following \cite{ssql:cao:ECCV2022}. Fig.~\ref{fig:voc} shows that our method outperforms BYOL on Pascal VOC at width of 1.0x. Also, as we decrease the width, our advantages over the baseline counterpart BYOL will be further expanded: up to \textbf{+2.6} and \textbf{+18.5} $\text{AP}_{50}$ at width of 0.5x and 0.25x, respectively. We can reach similar conclusions on COCO2017 from Table~\ref{tab:coco}. Although our method achieves comparable accuracy to BYOL at 1.0x on COCO, we achieve \textbf{+0.9} and \textbf{+1.5} $\text{AP}^\text{bb}$ points higher at 0.5x and 0.25x, respectively. Note that our improvements on COCO are not as large as that on VOC. It is because the amount of training data in COCO is large enough to close the gap between different pretrained models, as noted in~\cite{rethinkpretraining:he:ICCV19}.

We also transfer the ImageNet learned representations to 5 downstream recognition benchmarks in Table~\ref{tab:classification}. As seen, our method improves a lot on all recognition benchmarks (except for pets) under linear evaluation. 

\begin{figure}[t]
    \centering
    \includegraphics[width=0.9\columnwidth]{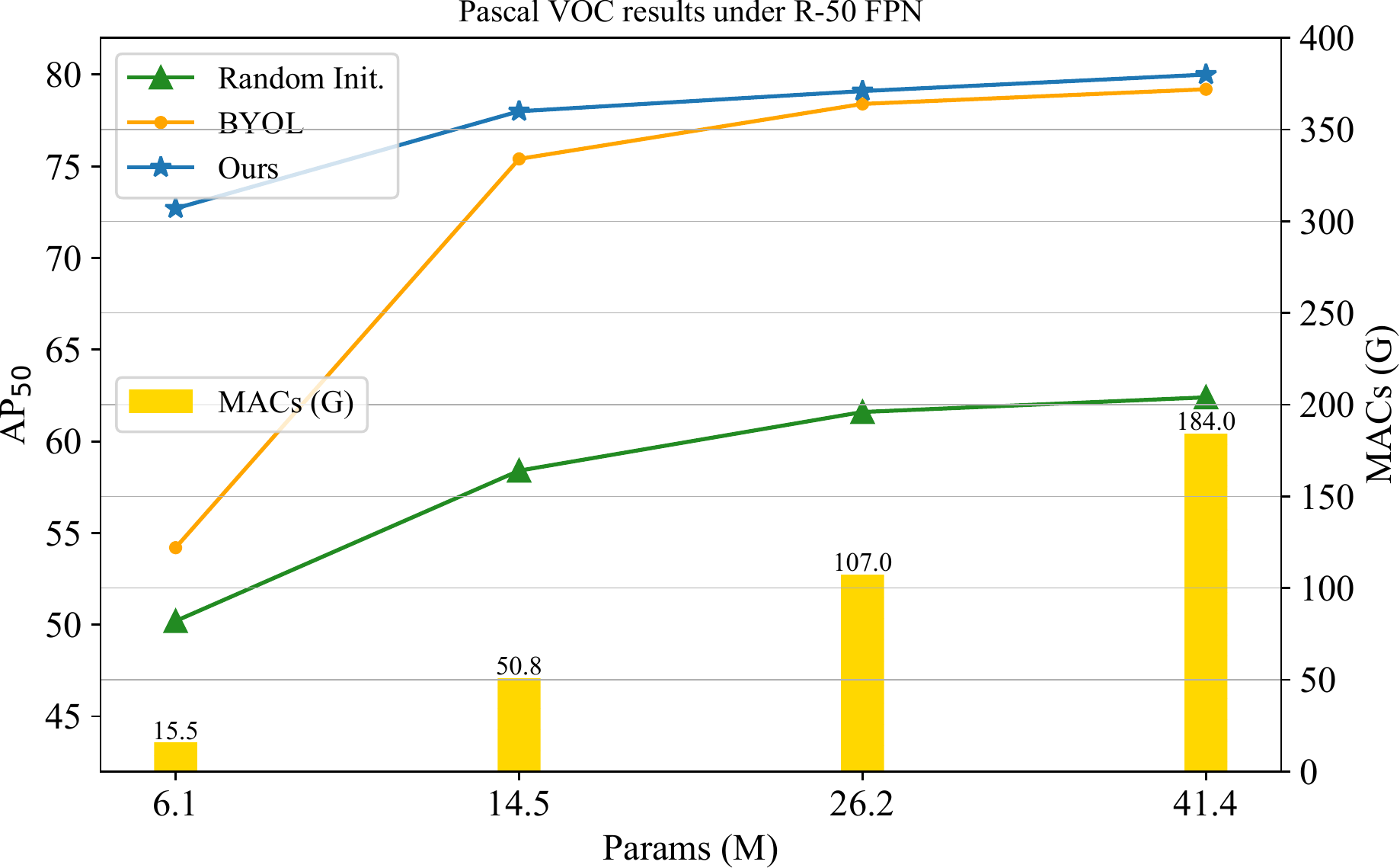}
    \caption{Transfer results on Pascal VOC 07\&12 under R50-FPN.}
    \label{fig:voc}
\end{figure}

\begin{table}[t]
	\centering
	\small
	\caption{Transfer results on COCO2017 object detection\& instance segmentation under R-50 FPN.}
	\label{tab:coco}
	\setlength{\tabcolsep}{2.5pt}
	\renewcommand{\arraystretch}{0.75}
	\renewcommand{\multirowsetup}{\centering}
	\begin{tabular}{c|c|c|c|c|c|c|c}
		\hline
		 Width & Method & $\text{AP}^{\text{bb}}$  &  $\text{AP}_{50}^{\text{bb}}$  &  $\text{AP}_{75}^{\text{bb}}$  &  $\text{AP}^{\text{mk}}$  &  $\text{AP}_{50}^{\text{mk}}$  & $\text{AP}_{75}^{\text{mk}}$  \\
		\hline
		\multirow{2}{*}{1.0x} & BYOL &  37.9 & 57.8 & 40.9 & 33.2 & 54.3 & 35.0 \\
		&Ours &\textbf{38.3} & \textbf{58.0} & \textbf{41.2} & \textbf{33.6} & \textbf{54.6} & \textbf{35.3} \\
		\hline
		\multirow{2}{*}{0.75x} & BYOL & 35.7 & 55.3 & 38.6 & 32.4 & 52.4 & 34.5  \\
		&Ours & \textbf{36.2} & \textbf{55.8} & \textbf{39.0} & \textbf{32.8} & \textbf{52.9} & \textbf{35.0} \\
        \hline
        \multirow{2}{*}{0.5x} & BYOL & 32.6 & 51.5 & 35.2 & 29.9 & 48.7 & 31.7   \\
		&Ours & \textbf{33.5} & \textbf{52.7} & \textbf{35.8} & \textbf{30.6} & \textbf{49.8} & \textbf{32.3} \\
		\hline
		 \multirow{2}{*}{0.25x} & BYOL & 26.0&43.5 & 27.0 & 24.2 & 40.8 & 25.3   \\
		 & Ours & \textbf{27.5} & \textbf{45.0} & \textbf{29.3} & \textbf{25.5} & \textbf{42.4} & \textbf{26.9} \\
		\hline
	\end{tabular}
\end{table}

\begin{table}[t]
	\centering
	\footnotesize
	\caption{Transfer results on recognition benchmarks under linear evaluation. `C-10/100' denotes `CIFAR-10/100'.}
	\label{tab:classification}
	\setlength{\tabcolsep}{1.5pt}
	\renewcommand{\arraystretch}{0.8}
	\renewcommand{\multirowsetup}{\centering}
	\begin{tabular}{l|c|c|c|c|c|c|c|c|c}
		\hline
		\multirow{2}{*}{Net}& \multirow{2}{*}{Width}& \multirow{2}{*}{Params}& \multirow{2}{*}{MACS}& \multirow{2}{*}{Method} & \multicolumn{5}{c}{Linear Accuracy (\%)} \\
		\cline{6-10}
		& &   & &&C-10 & C-100 & Flowers & Pets & Dtd\\
		\hline
        \multirow{8}{*}{R-50}
		&\multirow{2}{*}{1.0x}  &\multirow{2}{*}{22.56M}   &\multirow{2}{*}{4.11G}&BYOL & \textbf{87.1} & 60.6 & 81.0 & \textbf{80.9} & 70.7 \\
		&&&&Ours & \textbf{87.1} & \textbf{61.5}  & \textbf{90.6} &  79.4 & \textbf{72.6}\\
		\cline{2-10}
		&\multirow{2}{*}{0.75x}  &\multirow{2}{*}{14.77M}   &\multirow{2}{*}{2.34G}&BYOL & 83.6 & 52.8 & 82.9 & 74.2 & 66.3 \\
		&&&&Ours&\textbf{84.4} & \textbf{56.9} & \textbf{89.7} & \textbf{78.0} & \textbf{71.1} \\
		\cline{2-10}
		&\multirow{2}{*}{0.5x}  &\multirow{2}{*}{\phantom{0}6.92M}   &\multirow{2}{*}{1.06G}&BYOL & 80.6 & 52.0 & 74.8 & 75.0 & 65.9 \\
		&&&&Ours&\textbf{81.6} & \textbf{52.8} & \textbf{88.1} & \textbf{76.8} & \textbf{68.8} \\
		\cline{2-10}
    	&\multirow{2}{*}{0.25x}  &\multirow{2}{*}{\phantom{0}1.99M}   &\multirow{2}{*}{0.28G}&BYOL & 75.9 & 46.2 & 75.4 & 64.7 &  61.4 \\
		&&&&Ours & \textbf{78.9} &  \textbf{49.9}  & \textbf{84.4} & \textbf{74.0} & \textbf{64.9} \\
        \hline
	\end{tabular}
\end{table}

\subsection{Ablation Studies}\label{sec:ablation}
In this subsection, we demonstrate the effectiveness of the proposed three guidelines as well as the proposed strategies.

\begin{table*}[t]
    \caption{Ablation studies of the loss design under ResNet-18 on CIFAR-100. `-' denotes the model collapses.}
	\label{tab:ablation}
	\centering
	\small
	\setlength{\tabcolsep}{2.5pt}
	\renewcommand{\arraystretch}{0.7}
	\renewcommand{\multirowsetup}{\centering}
	\begin{tabular}{l|c|c|c|c|c|c|c|c|c|c|c|c|c|c}
		\hline
		\multirow{2}{*}{Base Loss}&\multirow{2}{*}{Case}&Distill& Auxiliary&\multicolumn{2}{c|}{Momentum Target} & \multicolumn{9}{c}{Linear Accuracy (\%)} \\
		\cline{5-15}
		&&Loss&Distill Head&Base Network&Sub Network&1.0x&0.9x&0.8x&0.7x&0.6x&0.5x&0.4x&0.3x&0.25x\\
		\hline
		\multirow{8}{*}{MSE}
		&1&$\times$&$\times$&$\times$&$\times$& - & - & - &-& -& -& -& -& - \\
		&2&MSE  &$\times$&$\times$&$\times$& 57.5&57.4&57.3&57.0&56.3&55.4&54.5&53.1&52.4  \\
		&3&MSE  &$\times$&\checkmark&\checkmark&    64.7&64.7&64.5&64.3&\textbf{63.9}&62.6&\textbf{61.3}&59.7&\textbf{59.3} \\
		&4&MSE  &\checkmark&\checkmark&\checkmark&  \textbf{65.4} & \textbf{65.0} & \textbf{64.8} & \textbf{64.5} & 63.8 & \textbf{62.7} & 61.1 & \textbf{59.8} & 58.9\\
		\cline{2-15}
		&5&InfoNCE&$\times$&$\times$&$\times$& 62.3&62.3&62.3&62.2&61.8&60.6&58.9&57.6&57.2 \\
		&6&InfoNCE&$\times$&$\times$&\checkmark& 63.7&63.8&63.7&63.6&63.1&62.0&60.6&59.3&58.2 \\ 
		&7&InfoNCE&$\times$&\checkmark&\checkmark&     65.0&65.0&65.1&\textbf{65.0}&64.5&62.7&61.3&59.8&59.2 \\
		&8&InfoNCE&\checkmark&\checkmark&\checkmark&  \textbf{65.5} &\textbf{65.5} &\textbf{65.6}& \textbf{65.0}&\textbf{64.6}&\textbf{63.2}&\textbf{61.6}&\textbf{60.2} & \textbf{59.7}\\
		\hline
		\multirow{9}{*}{InfoNCE}
		&9&$\times$&$\times$&$\times$&$\times$& 64.8 & 64.0 & 63.2 & 62.0 & 60.8 & 59.8	& 57.4 & 55.1 & 54.2 \\
		&10&MSE&$\times$&$\times$&$\times$& 65.0 &	64.4 & 63.1	& 62.3	& 61.9	& 60.3	& 58.3	& 57.1	& 56.6 \\
		&11&MSE&$\times$&$\times$&\checkmark& 65.8 &	65.0&	64.4&	63.4	&62.7	&61.8&	59.8&	58.5&	57.6 \\
		&12&MSE&$\times$&\checkmark&\checkmark&  66.9&	66.3&	65.7&	64.9	&63.8&62.9&	61.6&	59.5&	59.1 \\
		&13&MSE&\checkmark&\checkmark&\checkmark& \textbf{67.7}	& \textbf{67.2} &\textbf{66.5}&\textbf{66.0}	& \textbf{65.1} & 	\textbf{64.3} &	\textbf{62.5} &\textbf{60.5}&\textbf{59.6} \\
		\cline{2-15}
		&14&InfoNCE&$\times$&$\times$&$\times$& 65.5&	64.9&	63.8&	63.6&	62.7	&61.8&	60.2&	58.2&	57.4 \\
		&15&InfoNCE&$\times$&$\times$&\checkmark& 64.7&64.5	&64.0	& 63.6	&62.3	&61.4&	59.8&	58.4&	57.9  \\
		&16&InfoNCE&$\times$&\checkmark&\checkmark&  66.0 &	65.4	& 64.8	& 64.3	& 63.8	& 62.4	& 61.1	& 59.8	& 58.7  \\
		&17&InfoNCE&\checkmark&\checkmark&\checkmark& \textbf{67.4} &	\textbf{66.0}&	\textbf{66.1}&	\textbf{65.6} &	\textbf{64.7} &	\textbf{64.0} &	\textbf{62.2} &	\textbf{60.2} &	\textbf{59.5} \\
        \hline
	\end{tabular}
\end{table*}

\begin{table}[t]
	\caption{Ablation studies of our strategies on CIFAR-100.}
	\label{tab:ablation1}
	\centering
	\small
	\setlength{\tabcolsep}{2pt}
	\renewcommand{\arraystretch}{0.7}
	\renewcommand{\multirowsetup}{\centering}
	\begin{tabular}{c|c|c|c|c|c|c|c|c}
		\hline
		 \multirow{2}{*}{Backbone}& Dynamic  &Group & \multicolumn{6}{c}{Linear Accuracy (\%)} \\
		\cline{4-9}
		&Sampling & Reg. &1.0x &0.8x &0.6x&0.5x &0.3x&0.25x\\
		\hline
		\multirow{4}{*}{R-18}
		&$\times$ &$\times$&  67.7 & 66.5 & 65.1 & 	64.3 & 60.5 & 59.6 \\
		&\checkmark &$\times$ &68.6 &67.2 &65.5&64.6&60.7& 59.9\\
		&$\times$ &\checkmark & 68.6 & 67.3 & 65.5 & 64.4 & \textbf{60.9} & 60.1 \\
		&\checkmark & \checkmark  & \textbf{69.0} &\textbf{68.0} &\textbf{66.1}&\textbf{64.7} &\textbf{60.9} & \textbf{60.4}\\
        \cline{1-9}
		\multirow{4}{*}{R-50}
		&$\times$ &$\times$&  71.0  & 70.6  & 70.0 & 69.1   & 67.2 & 66.8 \\
		&\checkmark &$\times$ & 71.8  & 71.1   & 70.2 & 69.3 & 67.3   & 67.2 \\
		&$\times$ &\checkmark & 71.9 & 71.1 & 70.0 & 69.6 & \textbf{67.7} & \textbf{67.5}\\
		&\checkmark  & \checkmark  &  \textbf{72.6}   &  \textbf{71.5} &   \textbf{70.6} & \textbf{70.2} &  \textbf{67.7}   & 67.4 \\
		\cline{1-9}
		\multirow{4}{*}{MBv2}
		&$\times$ &$\times$&62.9 &62.0 & 61.5&60.4&59.6 &58.7  \\
		&\checkmark &$\times$ & 64.7 &  63.3 &  62.3 & 61.7 & \textbf{60.7} &   59.2 \\
		&$\times$ &\checkmark & 64.0 & 63.2 & 62.1 & 61.4&60.2&59.0 \\
		&\checkmark  & \checkmark  & \textbf{65.7}  & \textbf{64.2} &  \textbf{63.4} & \textbf{62.2} & \textbf{60.7}   & \textbf{59.3} \\
		\hline
	\end{tabular}
\end{table}

\subsubsection{Effectiveness of The Three Guidelines}
We investigate various combinations of training loss, distillation head and momentum target in Table~\ref{tab:ablation}. The `Base Loss' column represents the loss function for the base (i.e., largest) model. The `Distill Loss' column represents the distillation loss for sub-networks (`$\times$' means sub-networks use the same loss as the base network without using distillation). The `Auxiliary Distill Head' column indicates whether to use an additional head for distillation (i.e., $g(\cdot)$ in~\eqref{eq:full-loss}). The `Momentum Target' column indicates whether to maintain a momentum encoder of the base network, `Base/Sub Network' column represents whether the base/sub networks use the output of the momentum encoder as the target. We have the following observations from Table~\ref{tab:ablation}:

\squishlist
    \item As aforementioned, the model will collapse if we use individual MSE loss for sub-networks in case 1 (i.e., SimSiam). This phenomenon is \textit{completely different from the supervised case}, where each sub-network can be trained with cross entropy loss alone to achieve good results. We conjecture that the base and sub networks directly imitate their respective targets when using MSE, which will bring instability to the entire model, as analyzed in Sec.~\ref{sec:discussion}.
    \item Following US-Net, we use inplace distillation to guide sub-networks in case 2 and it solves the collapse problem. It is because now all sub-networks are aligned to the same target, which ensures cross-subnet consistency. Nevertheless, there is still a large gap compared with individually trained self-supervised networks (nearly 10 points).
    \item If we continue to use a momentum teacher, we will find consistent improvements in Table~\ref{tab:ablation} (e.g., case 3 vs. case 2). When comparing case 7 and case 6, we find that it is better to align all networks to the same target, which confirms our analysis of guidance consistency again. Consistency should not only exist between iterations, but also across sub-networks (all sub-networks should use the same target).
    \item \textbf{Experimental results are in full agreement with the proposed three guidelines}. Lemma~\ref{lemma1} states that InfoNCE can obtain a more stable gradient than MSE. Notice that the loss function consists of the base and distillation loss. (i) When the base loss uses MSE, the output of the base network will be unstable, so the sub-networks need to use InfoNCE loss for distillation to deal with this instability (case 5 v.s. case 2). (ii) When the base loss uses InfoNCE, the base network can achieve better temporal stability. Hence, the sub-networks already get stable targets from the teacher and using MSE or InfoNCE for distillation (case 10\&14) can achieve good results. In short, in the absence of a momentum teacher, at least one of base loss and distillation loss should use InfoNCE to ensure stability.
    \item The use of an auxiliary distillation head will result in consistent improvements when we compare the last two rows of each block in Table~\ref{tab:ablation} (e.g., case 8 vs. case 7).
\squishend

\subsubsection{Effectiveness of Our Strategies} 
We study the effect of our proposed strategies in Table~\ref{tab:ablation1} and the baseline is the best practice in Table~\ref{tab:ablation} (i.e., case 13). We can have the following conclusions from Table~\ref{tab:ablation1}:
\squishlist
    \item Our dynamic sampling strategy can improve the accuracy of the base network as well as sub-networks significantly. Although our total number of iterations is less (without training sub-networks at the start), we can still guarantee the performance of small sub-networks. It is because now we get a better and more stable teacher and then distillation speeds up the convergence of sub-networks.
    \item Our group regularization can also significantly improve the accuracy of the largest model while improving the accuracy of sub-networks. The experimental results verify our analysis in Sec.~\ref{sec:groupreg} that our grouping strategy can alleviate the problem of limited model capacity in US-Net.
\squishend

We also study the intersection of the sandwich rule and dynamic sampling in Table~\ref{tab:intersect}.
As seen, our dynamic sampling strategy can be used alone or combined with the sandwich rule, which brings improved performance for both scenarios. The results further demonstrate the universality and effectiveness of our dynamic sampling.

\begin{table}
    \centering
	\caption{Intersection of the sandwich rule and dynamic sampling on CIFAR-100 under ResNet-18.}
	\label{tab:intersect}
    \small
    \setlength{\tabcolsep}{2.5pt}
    \renewcommand{\arraystretch}{0.75}
    \renewcommand{\multirowsetup}{\centering}
    \begin{tabular}{c|c|c|c|c|c|c|c}
    	\hline
    	 Sandwich  &Dynamic & \multicolumn{6}{c}{Linear Accuracy (\%)} \\
    	\cline{3-8}
    	Rule & Sampling &1.0x &0.8x &0.6x&0.5x &0.3x&0.25x\\
    	\hline
    	$\times$ & $\times$ & 65.1 & 65.0 & 64.7 & 63.2 & 59.4&56.4 \\
    	$\times$ & \checkmark & 67.4 & \textbf{67.3} & \textbf{65.9} & \textbf{64.7} & 59.9 & 58.7 \\
    	\checkmark &$\times$&  67.7 & 66.5 & 65.1 & 	64.3 & 60.5 & 59.6 \\
    	\checkmark &\checkmark & \textbf{68.6} & \textbf{67.3} & 65.5 & 64.4 & \textbf{60.9} & \textbf{60.1} \\
    	\hline
    \end{tabular}
\end{table}

\section{Conclusions}
In this paper, we proposed a method called US3L for training universally slimmable self-supervised models. We provided theoretical analyses about the loss design and proposed three guidelines to ensure temporal consistency for US-Net training. Moreover, we proposed dynamic sampling and group regularization to solve the problems of inefficient training and limited model capacity. Experiments on various benchmarks and architectures (both CNNs and ViTs) show that our method significantly outperforms various baselines. When transferring to various downstream tasks, our models exhibit significant advantages at different widths, with only once training and one copy of weights. In the future, we will try to compress in dimensions such as depth and kernel size and explores combinations with NAS methods.

{\small
\bibliographystyle{ieee_fullname}
\bibliography{egbib}
}

\clearpage
\appendix

\section{Training Details}

\noindent\textbf{Datasets.} The statistics of the recognition benchmarks used in our paper are shown in Table~\ref{tab:dataset-overview}.

\begin{table}[!h]
    \caption{Statistics of the recognition benchmarks used in the paper.}	\label{tab:dataset-overview}
	\centering
	\footnotesize
	\renewcommand{\multirowsetup}{\centering}
	\begin{tabular}{l|c|c|c}
			\hline
			Datasets & \# Category & \# Training & \# Testing \\
			\hline 
			CIFAR-10 &10 & 50,000 & 10,000\\
			CIFAR-100 &100 & 50,000 & 10,000\\
			Flowers &102& 2,040&6,149\\
			Pets &\pt37& 3,680&3,669\\
			DTD &\pt47& 3,760&1,880\\
			\hline
	\end{tabular}
\end{table} 

\noindent\textbf{Training details for SSL methods.} Training details for SimCLR, BYOL, SimSiam and our US3L on CIFAR-100 are shown in Table~\ref{tab:appendix-details}.

\begin{table}[!h]
    \caption{Training details for SimCLR, BYOL, SimSiam and our US3L on CIFAR-100 in Table~\ref{tab:main-cifar100}, Table~\ref{tab:ablation} and Table~\ref{tab:ablation1}. $\tau$ denotes the temperature parameter, and $m$ denotes the momentum coefficient for the momentum network.}
	\label{tab:appendix-details}
	\centering
	\small
	\setlength{\tabcolsep}{1.5pt}
	\begin{tabular}{l|c|c|c|c|c|c|c|c|c}
		\hline
		\multirow{2}{*}{Method}  & \multicolumn{9}{c}{Settings}  \\
		\cline{2-10}
		&bs&lr&wd&epochs&optimizer&lr sche.&$\tau$&$m$&dim\\
		\hline
		SimSiam & 512 & 0.1 &5e-4 &400&SGD & cosine & -  & - & 2048 \\
		SimCLR   &512&0.5& 1e-4 & 400&SGD& cosine&0.5&-&2048\\ 
		BYOL   &512&0.1& 5e-4 & 400&SGD& cosine&-&0.99&2048 \\ 
		Ours & 512 & 0.5 & 1e-4 &400& SGD& cosine & 0.5 & 0.99 & 2048 \\
 		\hline
	\end{tabular}
\end{table}

\noindent\textbf{Training details for linear evaluation and fine-tuning.} For ImageNet linear evaluation, we follow the same settings in \cite{simsiam:kaiming:cvpr2021}. For linear evaluation on other datasets, we train for 100 epochs with lr initialized to 30.0, which is divided by 10 at the 60-th and 80-th epoch. 

\noindent\textbf{Source codes.} We promise that all codes will be made publicly available upon acceptance of the paper.

\section{More Results}
\subsection{Transfer Results for ResNet-18}
We plot the downstream object detection performance on Pascal VOC07\&12 for ResNet-18 FPN in Fig.~\ref{fig:voc-r18}. Moreover, we present the transfer results on downstream recognition benchmarks for ResNet-18 in Table~\ref{tab:classification-r18}. The results show that our method is also effective for ResNet-18 when transferring to downstream object detection and recognition tasks.

\begin{figure}[t]
    \centering
    \includegraphics[width=0.9\columnwidth]{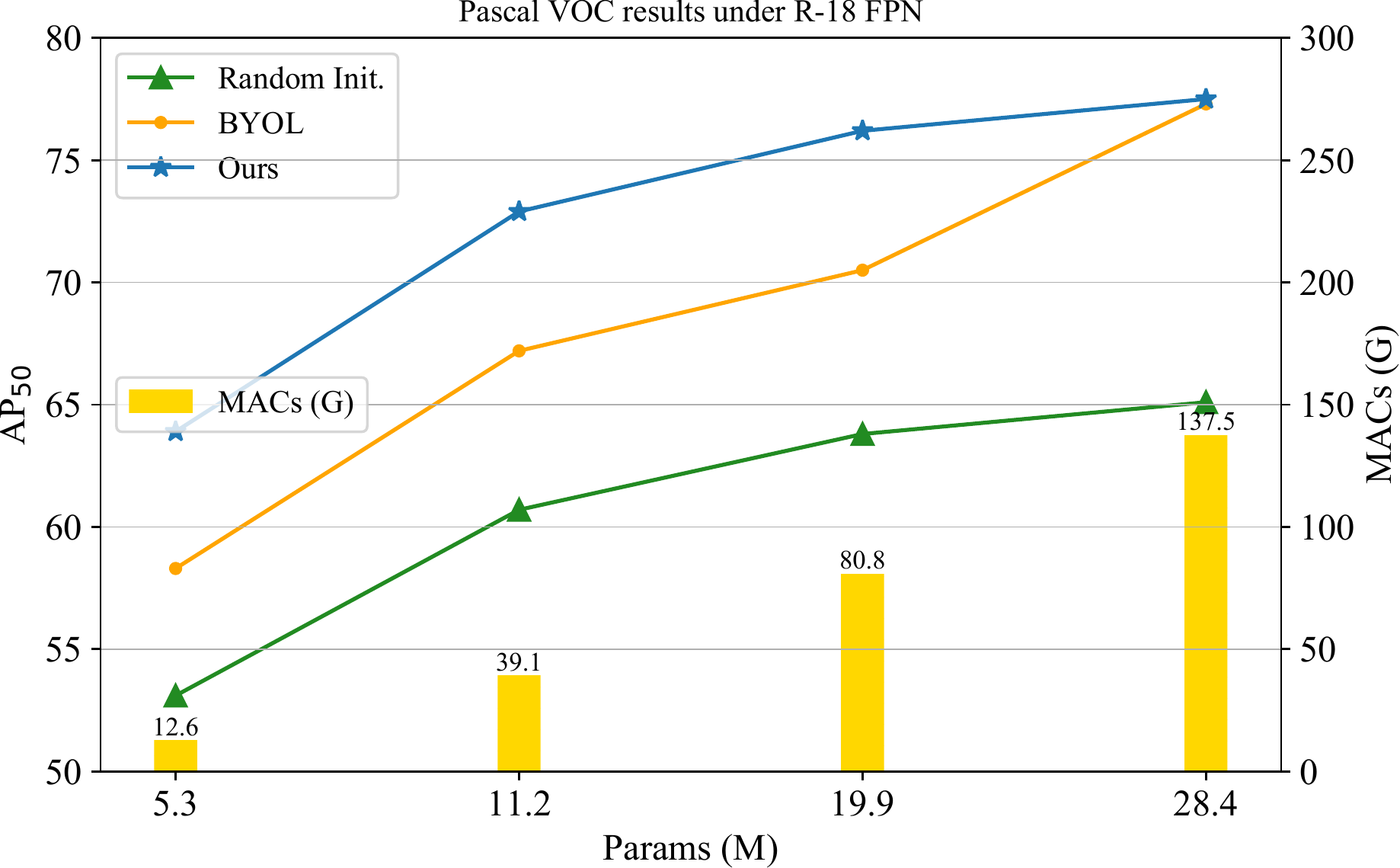}
    \caption{Transfer results on Pascal VOC 07\&12 under R18-FPN.}
    \label{fig:voc-r18}
\end{figure}

\begin{table}[t]
	\centering
	\footnotesize
	\caption{Transfer results on recognition benchmarks under linear evaluation. `C-10/100' denotes `CIFAR-10/100'.}
	\label{tab:classification-r18}
	\setlength{\tabcolsep}{1pt}
	\renewcommand{\arraystretch}{0.75}
	\renewcommand{\multirowsetup}{\centering}
	\begin{tabular}{l|c|c|c|c|c|c|c|c|c}
		\hline
		\multirow{2}{*}{Net}& \multirow{2}{*}{Width}& \multirow{2}{*}{Params}& \multirow{2}{*}{MACS}& \multirow{2}{*}{Method} & \multicolumn{5}{c}{Linear Accuracy (\%)} \\
		\cline{6-10}
		& &   & &&C-10 & C-100 & Flowers & Pets & Dtd\\
		\hline
        \multirow{8}{*}{R-18}
		&\multirow{2}{*}{1.0x}  &\multirow{2}{*}{11.69M}   &\multirow{2}{*}{1.82G}&BYOL & 76.6  &48.6 & 83.0 &  71.1 & 64.8 \\
		&&&&Ours &\textbf{77.9}&\textbf{52.6}  &\textbf{84.9}  &  \textbf{71.2} & \textbf{65.2} \\
		\cline{2-10}
		&\multirow{2}{*}{0.75x}  &\multirow{2}{*}{\phantom{0}6.68M}   &\multirow{2}{*}{1.05G}&BYOL &  76.4 & 48.1 &82.6 & \textbf{71.0} & \textbf{63.7} \\
		&&&&Ours & \textbf{76.7}   & \textbf{49.0}  & \textbf{83.5}  & \textbf{71.0} & \textbf{63.7} \\
		\cline{2-10}
		&\multirow{2}{*}{0.5x}  &\multirow{2}{*}{\phantom{0}3.06M}   &\multirow{2}{*}{0.49G}&BYOL & 74.6 & 46.9 &81.4&\textbf{67.6} & 61.7 \\
		& &&&Ours&\textbf{75.2} &\textbf{47.4} &\textbf{82.0} & 67.3 & \textbf{62.6} \\
		\cline{2-10}
    	&\multirow{2}{*}{0.25x}  &\multirow{2}{*}{\phantom{0}0.83M}   &\multirow{2}{*}{0.14G}&BYOL & 67.0 &41.2 &75.5&57.6&56.4 \\
		&&&&Ours & \textbf{67.8} & \textbf{41.4} &\textbf{77.0} & \textbf{60.6} & \textbf{56.6}\\
        \hline
	\end{tabular}
\end{table}

\subsection{Ablation Studies of Dynamic Sampling}
In this subsection, we conduct ablation studies of our dynamic sampling strategy and show how we successfully reduced the sampling number $s$ from 4 to 3 by using dynamic sampling while improving accuracy. We also compute the expected total forward number for $T$ iterations. Take our dynamic sampling as an example, we train the largest model only in the first $\frac{T}{4}$ iterations and sample three sub-networks in the last $\frac{3T}{4}$ iterations, hence the expected total forward number is:
\begin{equation}
    1\times \frac{T}{4} + 3\times \frac{3T}{4} = 2.5T \,.
\end{equation}
As shown in Table~\ref{tab:app-ablation-sample}, our dynamic sampling strategy achieves the best accuracy-efficiency trade-off. Notice that we also investigate the two components in our dynamic sampling strategy separately and we can clearly see that `max first' reduces the training overhead (case 4) and `gradually reduce' improves the accuracy (case 3).

\begin{table*}[htbp]
	\caption{Ablation studies of the sampling strategy under ResNet-18 on CIFAR-100. $T$ denotes the total number of iterations. `Max first' denotes whether to train the largest network only in early epochs. `Gradually reduce' denotes whether to gradually reduce the width of the smallest network.}
	\label{tab:app-ablation-sample}
	\centering
	\setlength{\tabcolsep}{1.5pt}
	\renewcommand{\multirowsetup}{\centering}
	\begin{tabular}{c|c|c|c|c|c|c|c|c|c|c|c|c|c}
		\hline
		\multirow{2}{*}{Case}&Sampling number& Expected forward & \multicolumn{2}{c|}{Dynamic Sampling}& \multicolumn{9}{c}{Linear Accuracy (\%)} \\
		\cline{4-14}
		&$s$&number & Max first & Gradually reduce  &1.0x&0.9x&0.8x&0.7x&0.6x&0.5x&0.4x&0.3x&0.25x\\
		\hline
		1& 4 & 4$T$ & $\times$ & $\times$ & 68.1 & 67.4 &67.0 & 66.3 &  65.3 & 64.4 & 62.7 & 60.8 & 59.9 \\
		2& 3 & 3$T$ & $\times$ & $\times$ & 67.7 & 67.2 & 66.5 & 66.0 & 65.1 & 64.3 & 62.5 & 60.5 & 59.6 \\
		3& 3 & 3$T$ &  $\times$ & \checkmark & 68.5 & 67.9 & \textbf{67.2} & 66.4 & 65.3 & 64.5 &  62.6 & \textbf{61.0} & \textbf{60.2} \\
		4& 3 & \textbf{2.5$T$} & \checkmark & $\times$ & 67.8 & 67.6 & 66.8 & 66.2 & 65.3 & 64.5 & 63.0 & 60.6 & 59.8 \\
		5& 3 & \textbf{2.5$T$}  & \checkmark & \checkmark & \textbf{68.6}&\textbf{68.1}&\textbf{67.2}&\textbf{66.6}&\textbf{65.5}&\textbf{64.6}&\textbf{62.8}&60.7&59.9 \\
		\hline
	\end{tabular}
\end{table*}

\subsection{Hyper-parameter Studies of $G$ and $\alpha$}
In this subsection, we study the choice of hyper-parameters $G$ and $\alpha$ in our group regularization (Eq.~\eqref{eq:greg_lambda}) in Table~\ref{tab:app-ablation-reg}. Notice that when $\alpha=0$, group regularization is equivalent to the standard $L_2$ normalization (case 1). We used $G=8$ and $\alpha=0.05$ in the paper. We also present the max decay fraction $G\times{\alpha}$ (i.e., the decay rate for the last group). Table~\ref{tab:app-ablation-reg} shows that we can achieve the best results when $G\times \alpha=40\%$ (case 1 $\sim$ case 5). Then we keep $G\times \alpha=40\%$ and change $G$ to 4 or 16. In terms of hyper-parameter $G$, we can see that $G=8$ (case 3) outperforms $G=4$ (case 6) and accuracy is saturated and will not continue to increase beyond 8 ($G=16$, case 7). It is worth noting that if we set $\alpha$ to a negative value which goes against our motivation (case 9), we will no longer see performance gains for large sub-networks as before. The results here further validate the effectiveness of our method and the correctness of our analysis in the paper.

\begin{table*}[t]
	\caption{Hyper-parameter studies of $G$ and $\alpha$ in group regularization under ResNet-18 on CIFAR-100.}
	\label{tab:app-ablation-reg}
	\centering
	\setlength{\tabcolsep}{2.5pt}
	\renewcommand{\multirowsetup}{\centering}
	\begin{tabular}{c|c|c|c|c|c|c|c|c|c|c|c|c|c|c}
		\hline
		\multirow{2}{*}{Case}& Max decay fraction&Group number & Decay rate & \multicolumn{11}{c}{Linear Accuracy (\%)} \\
		\cline{5-15}
		&$G\times \alpha$ &$G$ & $\alpha$  &1.0x&0.9x&0.8x&0.7x&0.6x&0.5x&0.4x&0.3x&0.25x& Avg. & 1.0x diff \\
		\hline
		1 & 0\% & - & 0  & 67.7 & 67.2 & 66.5 & 66.0 & 65.1 & 64.3 & 62.5 & 60.5 & 59.6 & 64.4 & +0.0\%\\
		2 & 20\% & 8 & 0.025 & 68.0 & 67.4 & 66.3& 66.0 & 65.2 & 64.0 & 62.6 &  61.0 & 60.3 & 64.5 & +0.3\% \\
 		3 & 40\%  & 8 & 0.05 & 68.6 & 67.8 & 67.3 & 66.4 & 65.5 & 64.4 & 63.1 & 60.9 & 60.1 & \textbf{64.9} & \textbf{+0.9\%} \\
		4 & 60\% &8 & 0.075  & 68.2 & 67.6 & 66.9 & 66.1 & 65.3 & 64.2 & 62.7 & 61.0 & 60.2 & 64.7 & +0.5\%\\
		5 & 80\% & 8 & 0.1  & 67.7 & 67.0 & 66.7 & 66.4 & 65.8 & 64.2 & 62.8 & 61.3 & 60.5 & 64.7 &  +0.0\%\\
		\hline
		6 & 40\% & 4 & 0.1  &  68.0 & 67.4 & 66.6 & 66.1 &  65.1 & 63.8 & 63.0 & 61.6 & 60.6 & 64.7& +0.3\%  \\
		7 & 40\% & 16 & 0.025 & 68.7 & 67.6 & 67.3 & 66.5 &  65.5 & 64.7 & 62.8 & 61.3 & 60.1 & \textbf{64.9} & \textbf{+1.0\%} \\
		8 & 80\% & 16 & 0.05  & 67.9 & 67.3 & 66.7 & 66.3 &  65.5 & 64.2 & 63.1 & 61.1 & 60.5 & 64.7 & +0.2\% \\
		\hline
		9 & -40\% & 8 & -0.05  & 67.4 & 67.0 & 66.3 & 65.9 & 64.7 & 64.0 & 62.8 & 61.2 & 60.1 & 64.4 & -0.3\% \\
		\hline
	\end{tabular}
\end{table*}

\begin{table*}[t]
	\caption{Effect of group regularization in BYOL and SimCLR under ResNet-18 on CIFAR-100. Our group regularization strategy is tailored for US-Net.}
	\label{tab:app-ablation-reg-for-byol}
	\centering
	\setlength{\tabcolsep}{2.5pt}
	\renewcommand{\multirowsetup}{\centering}
	\begin{tabular}{c|c|c|c|c|c|c|c|c|c|c|c|c}
		\hline
		\multirow{2}{*}{Method} &Model & Once&Group& \multicolumn{9}{c}{Linear Accuracy (\%)} \\
		\cline{5-13}
		&  Type& Training & Regularization &1.0x&0.9x&0.8x&0.7x&0.6x&0.5x&0.4x&0.3x&0.25x\\
		\hline
		\multirow{2}{*}{BYOL} & \multirow{2}{*}{individual} & \multirow{2}{*}{$\times$} & $\times$&66.8&	66.0	&65.6&	65.3&	63.0&	62.1	& 59.5 & 	56.0	&54.3 \\
		&&&\checkmark & 66.0 &	66.2&	65.6&	64.4	&63.4&	61.8&	59.3	&56.1&	54.0 \\
		\hline
		\multirow{2}{*}{SimCLR} & \multirow{2}{*}{individual} & \multirow{2}{*}{$\times$} & $\times$ & 66.5	&65.4&	64.7&	63.7	&62.6&	61.0&	59.0&	56.1&	53.6\\
		&&&\checkmark & 65.7&	65.0&	65.0&	63.3&	62.6&61.0 &	58.8	&56.0&	53.4\\
		\hline
		\multirow{2}{*}{Ours} & \multirow{2}{*}{US-Net} & \multirow{2}{*}{\checkmark} & $\times$ & 67.7 & 67.2 & 66.5 & 66.0 & 65.1 & 64.3 & 62.5 & 60.5 & 59.6\\
		&&&\checkmark&\textbf{68.6} & \textbf{67.8} & \textbf{67.3} & \textbf{66.4} & \textbf{65.5} & \textbf{64.4} & \textbf{63.1} & \textbf{60.9} & \textbf{60.1} \\
		\hline
	\end{tabular}
\end{table*}

\subsection{Group Regularization is Tailored for US-Net}
In this subsection, we will demonstrate that our group regularization strategy is tailored for US-Net and our analysis in the paper is valid. We apply group regularization to common SSL methods which train each model individually. As shown in Table~\ref{tab:app-ablation-reg-for-byol}, the introduction of group regularization does not bring improvements for BYOL and SimCLR, which are individually trained. It shows that the group regularization is tailored for US-Net, and the improvement is due to our unique design rather than factors such as hyper-parameters.

\subsection{Our US3L Can Run at Arbitrary Width}
Note that we only reported the results of width at [1.0, 0.9, 0.8, 0.7, 0.6, 0.5, 0.4, 0.3, 0.25]x for CIFAR-100 and [1.0, 0.75, 0.5, 0.25]x for ImageNet due to limited space. Actually, the pretrained model of our US3L can run at any width within the predefined width range, by only training once. As a supplement, we present the results of more widths on CIFAR-100 in Table~\ref{tab:app-arbitrary-width} and we can see that our pretrained model can achieve a good accuracy-efficiency trade-off. 

\begin{table*}[t]
	\caption{Results of our US3L method at different widths under ResNet-18 and ResNet-50 on CIFAR-100. Our US3L can run at arbitrary width and we only reported partial results as a representative in the paper due to limited space.}
	\label{tab:app-arbitrary-width}
	\centering
	\setlength{\tabcolsep}{1.5pt}
	\renewcommand{\arraystretch}{1.5}
	\renewcommand{\multirowsetup}{\centering}
	\begin{tabular}{c|c|c|c|c|c|c|c|c|c|c|c|c|c|c|c|c|c|c}
	\hline
	\multirow{2}{*}{Method}  & \multirow{2}{*}{Backbone}  & \multicolumn{17}{c}{Linear Accuracy (\%)} \\
	\cline{3-19}
	&&1.0x&0.95x&0.9x&0.85x&0.8x&0.75x& 0.7x&0.65x&0.6x&0.55x&0.5x&0.45x&0.4x&0.35x&0.3x&0.275x&0.25x\\
	\hline
	\multirow{2}{*}{Ours (800ep)} & R-18 & 70.1& 69.6 & 69.3 & 69.2 & 69.0 & 68.4 &  68.7 &68.0 &  67.3 &66.7 & 66.4 &65.4 & 64.2 &63.6 & 63.1 &63.1& 62.3\\
	\cline{2-19}
	& R-50 & 73.0&72.9& 72.5 & 72.1&71.9&71.6&71.6&71.2&71.1&71.0&70.8&69.9&69.1&68.3&68.0&67.8& 67.6\\
	\hline
	\end{tabular}
\end{table*}

\subsection{Figures}
As a supplement to Table~\ref{tab:main-cifar100} in the paper, we plot the results here to more intuitively see the advantages of our method. We compare with individually trained methods in Fig.~\ref{fig:compare-individual} and compare with the US-Net baseline in Fig.~\ref{fig:compare-us}.

\begin{figure*}[t]
    \centering
    \subfloat[ResNet-18]{
        \label{fig:individual-r18}
        \includegraphics[width=0.47\linewidth]{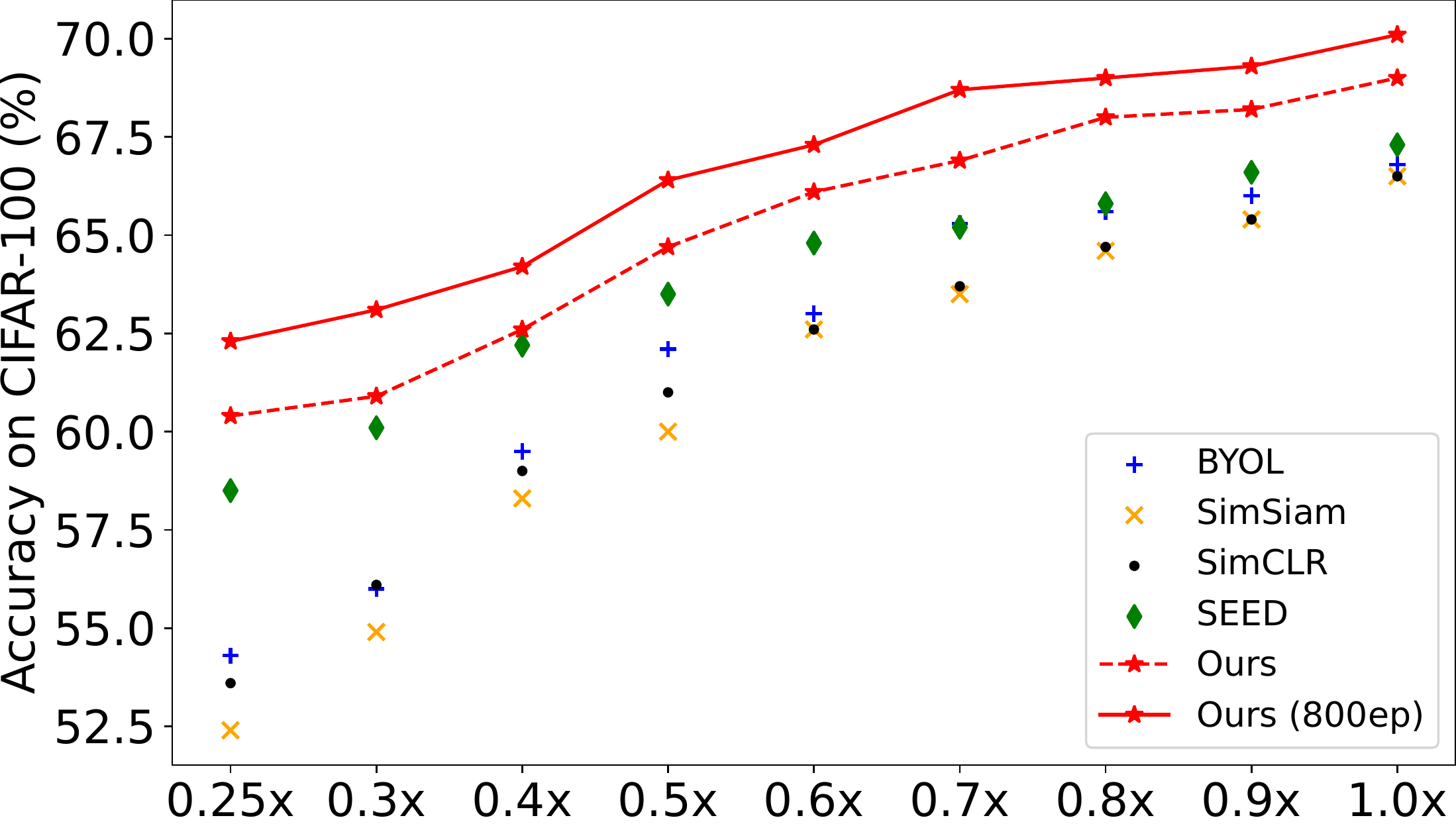}
	}
    \subfloat[ResNet-50]{
        \label{fig:individual-r50}
        \includegraphics[width=0.47\linewidth]{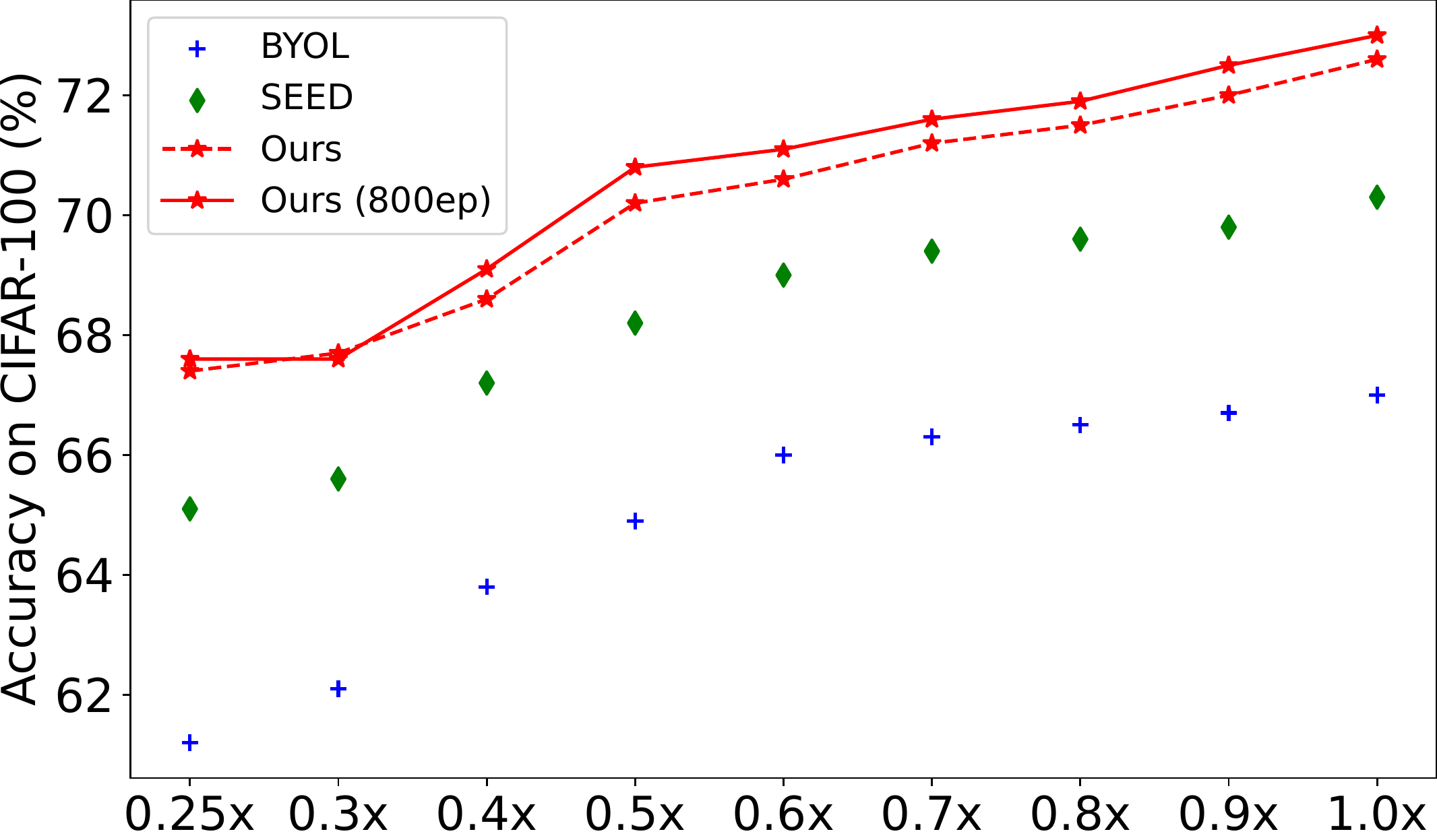}
	}
    \caption{Comparison with individually trained baselines on CIFAR-100. All scatters are individually trained, whereas our method is trained only once (the red line).}
    \label{fig:compare-individual}
\end{figure*}

\begin{figure*}[t]
    \centering
    \subfloat[ResNet-18]{
        \label{fig:us-r18}
        \includegraphics[width=0.47\linewidth]{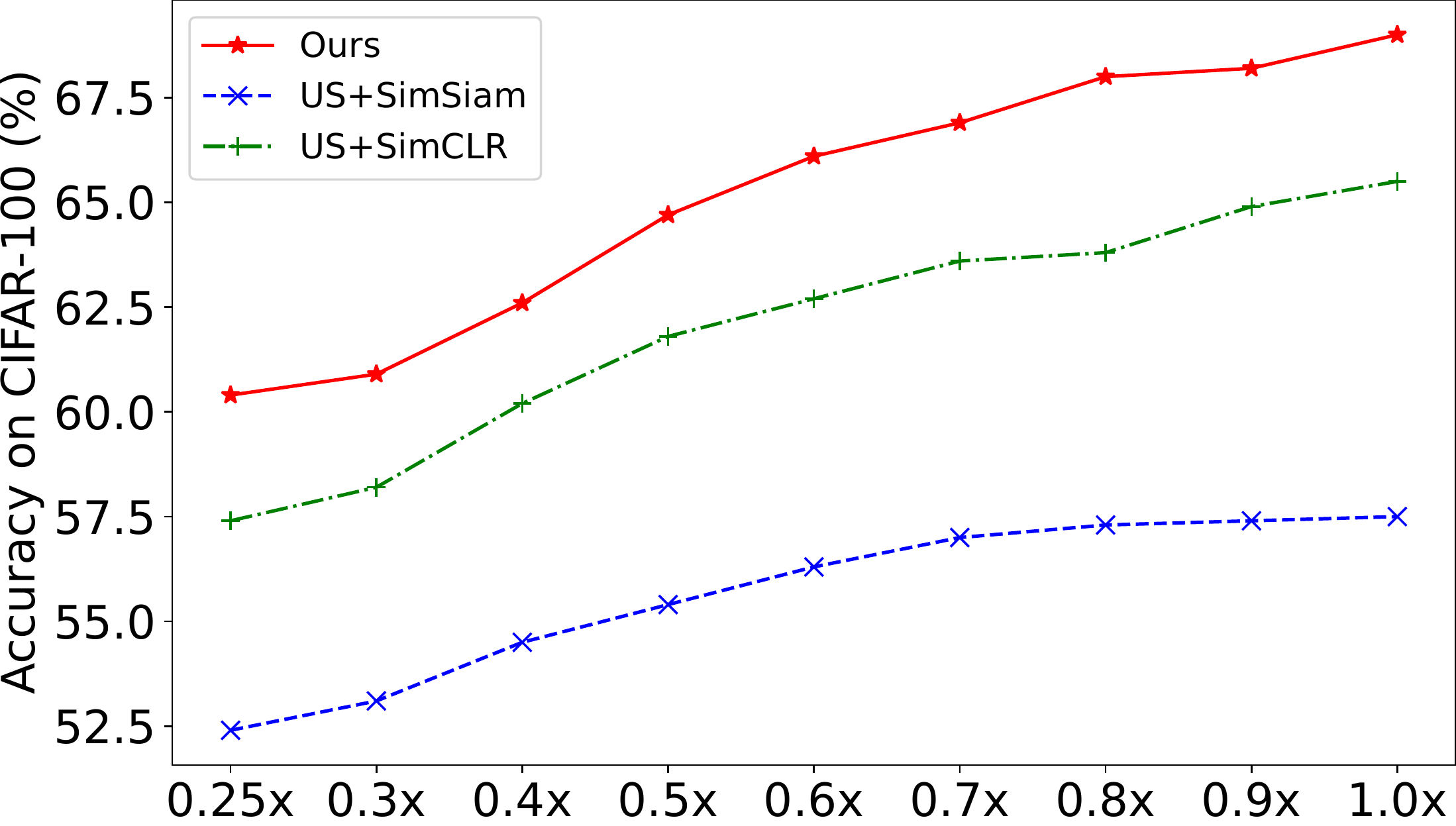}
	}
    \subfloat[ResNet-50]{
        \label{fig:us-r50}
        \includegraphics[width=0.47\linewidth]{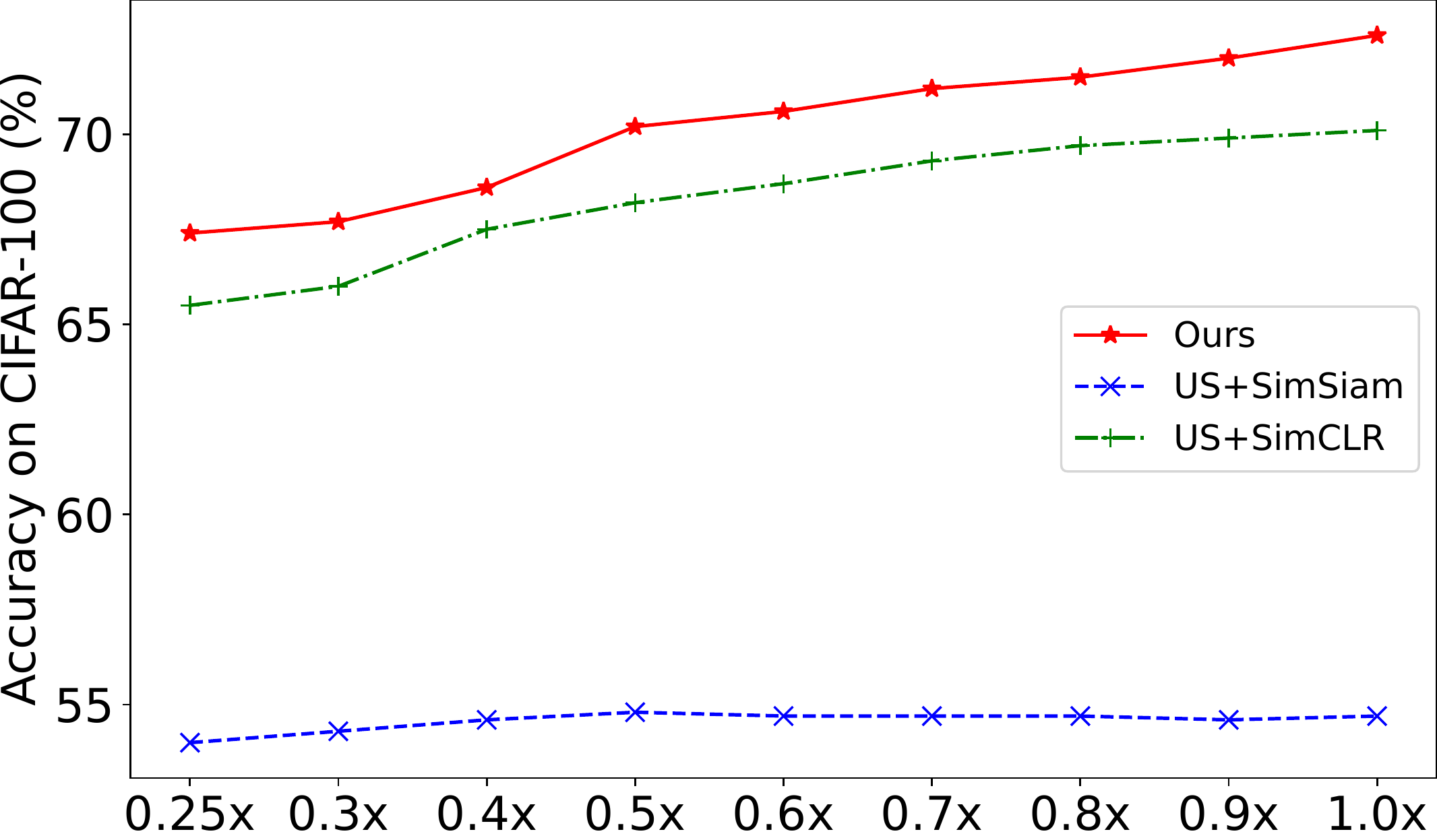}
	}
    \caption{Comparison with the original US-Net baseline on CIFAR-100. All are trained only once for 400 epochs.}
    \label{fig:compare-us}
\end{figure*}

\subsection{Ablation Studies of Loss Design}
We present more ablation results of loss design here in Table~\ref{tab:app-ablation}, as a supplement to Table~\ref{tab:ablation} in the paper. We look more closely at the `Asymmetric Distill Head' column, which indicates whether to use an additional head for distillation. Notice that there is already an asymmetrical head itself in MSE-based methods like SimSiam and BYOL. So `Share' refers to sharing the asymmetrical head, and `New' refers to distillation using a brand new head.

\begin{table*}[t]
	\caption{Ablation studies of the loss design under ResNet-18 on CIFAR-100. `-' denotes the model collapses.}
	\label{tab:app-ablation}
	\centering
	\small
	\setlength{\tabcolsep}{2.5pt}
	\renewcommand{\arraystretch}{0.7}
	\renewcommand{\multirowsetup}{\centering}
	\begin{tabular}{l|c|c|c|c|c|c|c|c|c|c|c|c|c|c}
		\hline
		\multirow{2}{*}{Base Loss}&\multirow{2}{*}{Case}&Distill& Asymmetric&\multicolumn{2}{c|}{Momentum Target} & \multicolumn{9}{c}{Linear Accuracy (\%)} \\
		\cline{5-15}
		&&Loss&Distill Head&Base model&Sub model&1.0x&0.9x&0.8x&0.7x&0.6x&0.5x&0.4x&0.3x&0.25x\\
		\hline
		\multirow{17}{*}{MSE}
		&1&$\times$&$\times$&$\times$&$\times$& - & - & - &-& -& -& -& -& - \\
		&2&MSE  &$\times$&$\times$&$\times$&- & - & - &-& -& -& -& -& -  \\
		&3&MSE  &\checkmark (Share)&$\times$&$\times$& 57.5&57.4&57.3&57.0&56.3&55.4&54.5&53.1&52.4  \\
		&4&MSE&\checkmark (Share)&$\times$&\checkmark&- & - & - &-& -& -& -& -& -  \\
		&5&MSE&$\times$&$\times$&\checkmark&- & - & - &-& -& -& -& -& -\\
		&6&MSE&$\times$&\checkmark&\checkmark&- & - & - &-& -& -& -& -& - \\
		&7&MSE  &\checkmark (Share)&\checkmark&\checkmark&    64.7&64.7&64.5&64.3&63.9&62.6&61.3&59.7&59.3 \\
		&8&MSE  &\checkmark (New)&\checkmark&\checkmark&  65.4 & 65.0 & 64.8 & 64.5 & 63.8 & 62.7 & 61.1 & 59.8 & 58.9\\
		\cline{2-15}
		&9&InfoNCE&$\times$&$\times$&$\times$& 62.3&62.3&62.3&62.2&61.8&60.6&58.9&57.6&57.2 \\
		&10&InfoNCE&\checkmark (Share)&$\times$&$\times$&58.7&58.8& 58.8&58.9&58.7&58.4&56.8&55.3&54.3 \\
		&11&InfoNCE&\checkmark (New)&$\times$&$\times$& 61.5 & 61.4 & 61.6 &  61.6 & 61.1 & 60.3 & 58.7 &  57.1 & 56.3 \\
		&12&InfoNCE&$\times$&$\times$&\checkmark& 63.7&63.8&63.7&63.6&63.1&62.0&60.6&59.3&58.2 \\ 
		&13&InfoNCE&\checkmark (Share)&$\times$&\checkmark&- & - & - &-& -& -& -& -& - \\
		&14&InfoNCE&\checkmark (New)&$\times$&\checkmark&64.5 & 64.5 & 64.6 & 64.5 & 64.2 & 63.2 & 62.1& 60.0&59.1  \\
		&15&InfoNCE&$\times$&\checkmark&\checkmark&     65.0&65.0&65.1&65.0&64.5&62.7&61.3&59.8&59.2 \\
		&16&InfoNCE&\checkmark (Share)&\checkmark&\checkmark&   65.0 & 64.9 & 64.9 & 64.4 & 64.1 & 62.8 & 61.1 & 60.0 & 59.5 \\
		&17&InfoNCE&\checkmark (New)&\checkmark&\checkmark&  \textbf{65.5} &\textbf{65.5} &\textbf{65.6}& \textbf{65.0}&\textbf{64.6}&\textbf{63.2}&\textbf{61.6}&\textbf{60.2} & \textbf{59.7}\\
		\hline
		\multirow{11}{*}{InfoNCE}
		&18&$\times$&$\times$&$\times$&$\times$& 64.8 & 64.0 & 63.2 & 62.0 & 60.8 & 59.8	& 57.4 & 55.1 & 54.2 \\
		&19&MSE&$\times$&$\times$&$\times$& 65.0 &	64.4 & 63.1	& 62.3	& 61.9	& 60.3	& 58.3	& 57.1	& 56.6 \\
		&20&MSE&$\times$&$\times$&\checkmark& 65.8 &	65.0&	64.4&	63.4	&62.7	&61.8&	59.8&	58.5&	57.6 \\
		&21&MSE&\checkmark&$\times$&\checkmark& 66.7 &	66.0	& 65.6	& 64.5&	63.3&	62.0	& 60.8	& 59.3	& 58.2 \\
		&22&MSE&$\times$&\checkmark&\checkmark&  66.9&	66.3&	65.7&	64.9	&63.8&62.9&	61.6&	59.5&	59.1 \\
		&23&MSE&\checkmark&\checkmark&\checkmark& \textbf{67.7}	& \textbf{67.2} &\textbf{66.5}&\textbf{66.0}	& \textbf{65.1} & 	\textbf{64.3} &	\textbf{62.5} &\textbf{60.5}&\textbf{59.6} \\
		\cline{2-15}
		&24&InfoNCE&$\times$&$\times$&$\times$& 65.5&	64.9&	63.8&	63.6&	62.7	&61.8&	60.2&	58.2&	57.4 \\
		&25&InfoNCE&$\times$&$\times$&\checkmark& 64.7&64.5	&64.0	& 63.6	&62.3	&61.4&	59.8&	58.4&	57.9  \\
		&26&InfoNCE&\checkmark&$\times$&\checkmark&  66.1&66.0&65.4&64.4&63.4&62.3&60.8&59.1&58.6 \\
		&27&InfoNCE&$\times$&\checkmark&\checkmark&  66.0 &	65.4	& 64.8	& 64.3	& 63.8	& 62.4	& 61.1	& 59.8	& 58.7  \\
		&28&InfoNCE&\checkmark&\checkmark&\checkmark& \textbf{67.4} &	\textbf{66.0}&	\textbf{66.1}&	\textbf{65.6} &	\textbf{64.7} &	\textbf{64.0} &	\textbf{62.2} &	\textbf{60.2} &	\textbf{59.5} \\
        \hline
	\end{tabular}
\end{table*}

\section{More Analysis}

\begin{lemma}
\label{lemma:n=3}
    $s=3$ is the theoretical minimum number of samples for US-Net~\cite{universal-slimmable:yu:ICCV19}.
\end{lemma}

\begin{proof}
    First, from \cite{universal-slimmable:yu:ICCV19} we know sandwich rules: Performances at all widths are bounded by the performance of the model at the smallest and the largest width. In other words, optimizing the lower and upper bounds of performance can implicitly optimize all sub-networks in a US-Net. To optimize for arbitrary widths, we need at least one randomly sampled width per iteration, except for the largest and smallest sub-networks. In conclusion, $s=3$ is the theoretical minimum number of samples for US-Net. 
\end{proof}

\end{document}